\documentclass{article}

\PassOptionsToPackage{numbers, compress}{natbib}

 \usepackage[preprint]{neurips_2026}

\usepackage{xcolor}
\usepackage[utf8]{inputenc} 
\usepackage[T1]{fontenc}    
\usepackage{amsmath, amssymb, amsthm}
\usepackage{algorithm}
\usepackage{algorithmic}
\usepackage[colorlinks=true, citecolor=blue]{hyperref}
\usepackage{mathtools}
\usepackage{graphicx}
\newtheorem{assumption}{Assumption}
\newtheorem{theorem}{Theorem}
\newtheorem{lemma}{Lemma}
\newtheorem{corollary}{Corollary}

\newtheorem{definition}{Definition}

\newtheorem{remark}{Remark}

\DeclareMathOperator{\E}{\mathbb{E}}
\DeclareMathOperator{\Prob}{\mathbb{P}}

\DeclarePairedDelimiter{\abs}{\lvert}{\rvert}

\newcommand{\R}{\mathbb{R}}

\newcommand{\Ocal}{\mathcal{O}}

\newcommand{\F}{\mathcal{F}}
\newcommand{\eps}{\varepsilon}

\newcommand{\ind}{\mathbf{1}}

\title{Dimension-Free Saddle-Point Escape in Muon}

\author{%
  Yanlin Long, Yufei Gu, Zeke Xie\thanks{Corresponding author: Zeke Xie (\texttt{zekexie@hkust-gz.edu.cn}). All authors are with xLeaF Lab, The Hong Kong University of Science and Technology (Guangzhou). Emails: \texttt{ynlinlong@gmail.com}, \texttt{ygu167@connect.hkust-gz.edu.cn}.}
}

\begin{document}

\maketitle

\begin{abstract}
Modern Large Language Model (LLM) training is fundamentally bottlenecked by pathologically flat saddle points in extreme high-dimensional landscapes. Motivated by this challenge, we analyze the saddle-point escape dynamics of the emerging Muon optimizer, demonstrating its resilience against the $\mathcal{O}(D)$ dimensional curse that severely traps element-wise adaptive optimizers like AdamW.  By extending generalized matrix perturbation theory, we develop a theoretical framework to capture Muon's non-equilibrium optimization trajectories. This theoretical machinery mathematically proves that Muon elegantly bypasses the dimensional curse via a non-linear spectral shaping mechanism. By leveraging resolvent functional calculus and macroscopic Cauchy contour integration, we avoid isotropic noise assumptions and Tracy-Widom edge singularities. We establish that structural incoherence securely shields the trajectory from orthogonal drift, enabling a dimension-free saddle-point escape, and triggering a deterministic $\mathcal{O}(1)$ discrete ballistic ejection under sufficient spectral gap. Consequently, we provide an algebraically dimension-free escape bound for Muon, formalizing the underlying mechanics of its non-convex optimization dynamics.
\end{abstract}

\section{Introduction}
\label{sec:intro}

The advent of Large Language Models (LLMs) has ushered in an era of massive over-parameterization, driven by established scaling laws \citep{kaplan2020scaling}. This paradigm shift fundamentally alters the topological bottleneck of neural network training: in extreme high-dimensional non-convex landscapes, optimization trajectories are no longer trapped by local minima, but rather obstructed by a proliferation of pathological saddle points \citep{dauphin2014identifying, choromanska2015loss}. This challenge is acutely magnified in modern Transformer architectures. Recent theoretical insights reveal that without structural interventions, pure attention matrices suffer from doubly exponential rank collapse \citep{dong2023attention}. Concurrently, the Zipfian distribution of natural language ensures that gradient signals for long-tailed tokens are remarkably sparse, injecting severe anisotropic noise into the optimization process \citep{zhu2019anisotropic, xie2023overlooked}. Together, these physical phenomena create uniquely hostile optimization terrains: ultra-high dimensional, pathologically flat saddle points characterized by massive rank deficiency and microscopic structural signals hidden within severe heteroskedastic noise bulks.

Escaping such non-convex traps has been a central focus of optimization theory. Classical analyses established that injecting stochastic noise enables Stochastic Gradient Descent (SGD) to escape saddle points and exponentially favor flat minima \citep{ge2015escaping, xie2021diffusion}. This was formalized by Perturbed Gradient Descent (PGD) \citep{jin2017how}, which mathematically guaranteed escape in polynomial time. However, a critical limitation persists: the escape time bounds for these isotropic perturbation methods scale polynomially with the ambient dimension $D$. In the LLM era, where matrix dimensions easily exceed $D \ge 8192$, this dimension dependence constitutes a fatal theoretical bottleneck.

To navigate ill-conditioned curvatures, adaptive methods like AdamW \citep{staib2020escaping, xie2022adaptive} and momentum-based techniques \citep{oneill2018behavior, xie2021positive} have become the industry standard. Yet, element-wise adaptive optimizers rely on coordinate-wise variance scaling. When confronted with the aforementioned flat, highly anisotropic terrains, the minuscule gradient signals of long-tailed features are overwhelmingly diluted by the accumulated moving average of batch noise. As we analyze in this paper, this forces AdamW into a state of multidimensional Brownian motion, trapping it in an extreme residence time of $\Theta(D/\lambda^2)$.

To break the coordinate-system limitation, matrix preconditioning techniques like Shampoo \citep{gupta2018shampoo,anil2021scalable} emerged, leveraging full-matrix inverse roots to reshape the spectral landscape. While theoretically powerful, computing and maintaining these inverse-root preconditioners remain computationally and memory intensive for ultra-large matrices. The recently proposed Muon optimizer \citep{jordan2024muon,liu2025muon} elegantly bypasses this bottleneck via a computationally cheap 5th-order Newton-Schulz polynomial iteration, offering a robust non-linear orthogonalization operator. 

While recent foundation model technical reports have widely celebrated Muon's convergence speed in modern LLM pretraining \citep{kteam2025kimik2,kovalev2025understanding}, the underlying physical mechanisms remain largely elusive. Recent theoretical efforts interpret its efficacy through various task-specific or static lenses—including associative memory \citep{wang2025muon, zhang2026muon}, variance reduction \citep{xu2025convergence}, and implicit Newton preconditioning \citep{jordan2024muon, du2026newton}—yet lack a kinematic formulation to capture its dynamic non-equilibrium trajectory.

\textbf{Our Contributions.} In this paper, we establish a kinematic formulation of the Muon optimizer's non-equilibrium behavior. By deploying generalized rectangular Random Matrix Theory, we offer the following theoretical and empirical advancements: \\
\textbf{(1) Demystifying the Dimensional Curse of AdamW and Muon's Algebraically Dimension-Free Escape (Theorems \ref{thm:subspace_locked_escape} and \ref{thm:comparative_adam_failure}):} We construct a generalized spiked rectangular matrix model to formally characterize the non-convex optimization landscape. Within this framework, we prove the dimension-free saddle-point escape capability of the Muon optimizer. Concurrently, we demonstrate that the escape velocity of standard adaptive methods, such as AdamW, is dependent on both the ambient dimensionality and the local landscape flatness. \\
\textbf{(2) Methodological Innovations in Random Matrix Theory (Theorems \ref{thm:subgaussian_perturbation} and \ref{thm:cross_step_immunity}):} A fundamental challenge in analyzing the Muon optimizer lies in handling orthogonal drift during the optimization process. To address this, we extend the $l_{2,\infty}$ singular subspace perturbation bounds \citep{wang2026analysis} beyond Gaussian assumptions to accommodate sub-Gaussian noise. Then, we develop a tailored resolvent functional calculus framework alongside macroscopic Cauchy contour integration. By elevating the cross-step analysis to the operator level, we utilize this novel Random Matrix Theory (RMT) toolkit to mathematically prove that Muon's update strictly nullifies orthogonal drift.\\ 
\textbf{(3) Empirical Observation Across Synthetic and Real-World Terrains:} Following the empirical evaluation of our theoretical framework in controlled synthetic environments, we demonstrate Muon's saddle-point escape capabilities within classical matrix factorization tasks. Through localized landscape probing during the pre-training of a LLaMA-160M model, we successfully observe key physical phenomena in real-world settings: Muon achieves a lower optimization loss while simultaneously elevating the effective rank and actively suppressing macroscopic gradient spikes.

\section{Related Work}
\label{sec:related_work}

\textbf{Saddle Point Escape and Dimensional Trapping.} Classical isotropic noise injection (e.g., PGD) \citep{jin2017how} and element-wise adaptive methods like AdamW \citep{staib2020escaping, zhu2019anisotropic} suffer from escape times that scale heavily with the ambient dimension $\text{poly}(D)$. While recent Correlated Negative Curvature (CNC) \citep{daneshmand2018escaping} and advanced variance-reduction theories \citep{reddi2018generic, zhou2020stochastic} establish dimension-free or dimension-benign escape bounds, they typically mandate explicit stochastic perturbation injection or rely strictly on specific finite-sum architectural structures. In contrast, our analysis isolates an intrinsic, non-linear \textit{spectral shaping} mechanism within Muon that theoretically achieves strict dimensional decoupling without requiring any explicit external micro-perturbations.

\textbf{Matrix Preconditioning vs. Polynomial Iteration.} To navigate ill-conditioned curvature, matrix preconditioning techniques like Shampoo \citep{gupta2018shampoo} and K-FAC \citep{osawa2019large} approximate local Hessian information to reshape the spectral landscape. However, despite their theoretical efficacy in handling pathological curvatures, these second-order approximations mandate computationally prohibitive matrix inversions or inverse roots. The Muon optimizer \citep{jordan2024muon} synthesizes these high-order preconditioning benefits while elegantly bypassing the inversion bottleneck entirely via lightweight Newton-Schulz polynomial iterations. 

\textbf{RMT and Spectral Optimizers.} While Random Matrix Theory (RMT) has successfully mapped the \textit{static} geometry of high-dimensional loss surfaces \citep{pennington2017geometry}, we extend generalized $l_{2,\infty}$ subspace perturbation bounds \citep{wang2026analysis} to decode \textit{dynamic} non-equilibrium optimization trajectories. Concurrently, while recent algorithmic variants of spectral optimizers (e.g., CANS \citep{CANS}, PRISM \citep{PRISM}, Mousse \citep{Mousse}, and ROOT \citep{ROOT}) primarily focus on numerical stability, coefficient tuning, and scaling efficiency, our work provides an orthogonal theoretical foundation. We strictly decode the kinematic phase transitions that mathematically define these polynomial algorithms under idealized conditions.



\section{Theoretical Results}
In this section, we establish the theoretical foundations of the Muon optimizer's non-equilibrium escape dynamics. We first formulate the generalized landscape assumptions and present the primary theoretical guarantees that formally decouple Muon's escape velocity from the ambient dimension, contrasting it with the inherent dimensional trapping of AdamW (Section \ref{subsec:main_guarantees}). Following this, we provide a detailed technical overview of the methodological breakthroughs in Random Matrix Theory and SVD functional calculus used to derive these escape dynamics (Section \ref{subsec:technical_overview}).

\subsection{Main Theoretical Guarantees}
\label{subsec:main_guarantees}

\begin{algorithm}[ht]
\small
\caption{Kimi-Variant Muon Optimizer (Core dynamics)\citep{liu2025muon}}
\label{alg:muon_Kimi_exact}
\begin{algorithmic}[1]
\REQUIRE Gradient $G_t$, Momentum $B_{t-1}$, Learning rate $\eta$, Momentum factor $\mu$, Dimension $d$
\REQUIRE Polynomial coefficients: $a=3.4445, b=-4.7750, c=2.0315$
\STATE $B_t \leftarrow \mu B_{t-1} + G_t$ \COMMENT{Momentum accumulation}
\STATE $X_0 \leftarrow \frac{B_t}{\|B_t\|_F}$ \COMMENT{Global Frobenius norm scaling}
\FOR{$k = 0$ \TO $4$}
    \STATE $X_{k+1} \leftarrow a X_k + b (X_k X_k^\top) X_k + c (X_k X_k^\top)^2 X_k$ \COMMENT{Newton-Schulz iteration}
\ENDFOR
\STATE $W_{t+1} \leftarrow W_t - 0.2 \eta \sqrt{d} X_5$ \COMMENT{Dimensional compensation injection}
\end{algorithmic}
\end{algorithm}

As shown in algorithm \ref{alg:muon_Kimi_exact}, the momentum coefficient $\mu \in (0,1)$ and the learning rate $\eta > 0$ are assumed to be sufficiently small. The polynomial coefficients are empirically fixed to $(a,b,c)=(3.4445,-4.7750,2.0315)$.

\begin{assumption}[Local Geometry, Generalized Heteroskedastic Noise, and Incoherence]
\label{assum:main_topology}
To capture the extreme heterogeneity of modern LLM landscapes without relying on idealized isotropic noise limits, our theoretical derivation relies on the following generalized geometric and statistical constraints:

    \textbf{Local Geometry and Structural Signal.} We assume the existence of a saddle point $W^*$ such that within a macroscopic neighborhood bounded by $r_0 > 0$, the instantaneous gradient field decomposes into a structural drift and stochastic noise: $G_t = S_t + E_t$. The structural signal $S_t \in \mathbb{R}^{m \times n}$ represents the localized deterministic gradient driven by the negative curvature $-\lambda$ ($\lambda > 0$), admitting a low-rank Singular Value Decomposition (SVD) $S_t = U \Sigma V^\top$ with rank $r \ll \min(m,n)$. Under standard parameterization (SP), we explicitly acknowledge that the spectral Lipschitz constant diverges with dimension, i.e., $L_{spec} = \Theta(\sqrt{d})$.

     \textbf{Generalized Sub-Gaussian Noise and Bounded Operator Norms.} Unlike classical theories that enforce a strict isotropic baseline covariance to induce a Tracy-Widom edge, we model the ambient noise $E_t \in \mathbb{R}^{m \times n}$ as a generalized zero-mean sub-Gaussian random matrix. This accommodates the severe spatial variance typical of modern architectures. While strictly heavy-tailed Zipfian distributions fall outside this sub-Gaussian envelope (a boundary we explore empirically in Section \ref{sec:ablation_noise}), our generalized framework models the baseline heteroskedasticity without mandating uniform isotropic bulk variance. We assume there exist finite bounds $B > 0$ and $L > 0$ such that, with high probability $1-\epsilon$, the noise matrix satisfies a global operator norm bound $\|E_t\|_{op} \le B$ and a projected noise concentration bound $\|U^\top E_t V\|_{op} \le L$. Crucially, we do not mandate a uniform bulk variance.

     \textbf{Structural Incoherence.} In ultra-high dimensional optimization, the absolute spectral gap ($\delta = \sigma_1 - \sigma_2$) is frequently submerged by the ambient noise norm ($B \gg \delta$), rendering classical Wedin's bounds vacuous. To mathematically isolate the structural signal $S_t$, we assume its singular vector matrices possess an incoherence parameter bounded by $\mu_0$:
    \[
    \|U\|_{2,\infty} = \max_i \|e_i^\top U\|_2 \le \frac{\mu_0}{\sqrt{m}}, \quad \|V\|_{2,\infty} \le \frac{\mu_0}{\sqrt{n}}.
    \]
    This parameterizes the delocalization of the topological signal. It mathematically guarantees that the structural curvature is not spuriously concentrated on a single pathological coordinate, thereby allowing the non-linear polynomial to dynamically align with the subspace regardless of the noise heteroskedasticity.

\end{assumption}

\begin{theorem}[Sub-Gaussian $l_{2,\infty}$ Singular Subspace Perturbation]
\label{thm:subgaussian_perturbation}
To evaluate the optimization dynamics under Assumption \ref{assum:main_topology}, we extend the generalized random perturbation framework \citep{wang2026analysis} to accommodate sub-Gaussian matrices. 

Let the gradient noise $E_t \in \mathbb{R}^{m \times n}$ consist of independent, zero-mean sub-Gaussian entries. Let $S_t = U \Sigma V^\top$ be the rank-$r$ deterministic structural signal with incoherence parameter $\|U\|_{2,\infty} \le \mu_0 / \sqrt{m}$. By incorporating the sub-Gaussian isotropic local law \citep{knowles2013isotropic}, the resolvent $G(z) = (zI - \mathcal{E}_t)^{-1}$ tightly concentrates around a deterministic diagonal matrix. Consequently, the principal singular subspace $\tilde{U}$ of the perturbed gradient $G_t = S_t + E_t$ satisfies, with high probability:
\[
\min_{O \in \mathbb{O}^{r \times r}} \|\tilde{U} - U O\|_{2,\infty} \le \mathcal{C} \left( \frac{\|E_t\|_{op}}{\delta} \frac{\mu_0}{\sqrt{m}} \right),
\]
where $\delta$ is the operational spectral gap and $\mathcal{C}$ is a constant governed by the sub-Gaussian norm of the noise. 
\end{theorem}

\begin{remark}[Methodological Extension to Sub-Gaussian Noise]
\label{rem:subgaussian_extension}
Recent breakthroughs in Random Matrix Theory have established optimal $l_{2,\infty}$ perturbation bounds for singular subspaces. However, these foundational results explicitly require Gaussian noise to exploit strict rotational invariance. The severe spatial variance characteristic of LLM gradient fields (Assumption \ref{assum:main_topology}) fundamentally violates this requirement. To operationalize these bounds for optimization dynamics, we provide a necessary theoretical extension. As hypothesized by \citep{wang2026analysis}, we bypass the Gaussian requirement by substituting rotational invariance with generalized sub-Gaussian concentration inequalities (specifically, the Hanson-Wright inequality) within the resolvent method. This extension, detailed in Appendix \ref{sec:appendix_subgaussian_proof}, bridges the gap between optimal RMT bounds and generalized non-convex optimization terrains.
\end{remark}

To frame the non-equilibrium evolution of the momentum matrix, we formalize the optimization boundary via a dual-stopping-time framework.

\begin{definition}[Dual-Phase Subspace Alignment Framework]
\label{def:stopping_times}
Let $\{W_t\}_{t \ge 0}$ be the non-linear optimization trajectory. We replace classical isotropic Tracy-Widom thresholds with a dynamic alignment framework governed by structural incoherence. We define two distinct temporal thresholds:

     \textbf{Geometric Escape Stopping Time ($\tau_{\text{esc}}$):} The first hitting time when the spatial parameter deviation strictly breaches the local Taylor approximation radius: $\tau_{\text{esc}} = \inf \{ t \ge 0 : \|W_t - W^*\|_F \ge r_0 \}$.\\
     \textbf{Subspace Lock Ignition Time ($\tau_{\text{lock}}$):} The first hitting time when the deterministic drift signal establishes a sufficient spectral gap relative to the local projected noise, triggering the non-linear polynomial to aggressively suppress transverse angular drift. Mathematically, it is the moment the effective structural drift overcomes the incoherence-scaled sub-Gaussian noise bound: $\tau_{\text{lock}} = \inf \{ t \ge 0 : \|D_t\|_{op} \ge \mathcal{C} \|E_t\|_{op} \frac{\mu_0}{\sqrt{m}} \}$.

\end{definition}

\begin{remark}[The Topographic Dichotomy of Escape]
The non-linear dynamics of the Muon optimizer within the local neighborhood are governed by the interplay between the geometric escape time ($\tau_{\text{esc}}$) and the subspace lock ignition time ($\tau_{\text{lock}}$). Extreme high-dimensional or pathologically flat landscapes often naturally suppress the dynamic subspace alignment ($\tau_{\text{esc}} < \tau_{\text{lock}}$). In this pre-locking regime, singular vectors undergo strict delocalization, yielding purely diffusive energy. Conversely, when local curvature is sufficient to ignite the alignment ($\tau_{\text{lock}} < \tau_{\text{esc}}$), the system undergoes a sudden macroscopic phase transition characterized by severe eigenvector localization with the structural drift. Consequently, the algorithm exhibits a dichotomous but universally dimension-resilient behavior.
\end{remark}

The valid operational domain for our localized stochastic differential analysis is strictly constrained to the interval $t < \min(\tau_{\text{esc}}, \tau_{\text{lock}})$. Under this framework, we present the main theoretical results demonstrating Muon's provable dimensional resilience.

\begin{theorem}[Dimension-Free Incubation and Subspace-Locked $\mathcal{O}(1)$ Ballistic Ejection via Polynomial Shaping]
\label{thm:subspace_locked_escape}
Suppose Assumption \ref{assum:main_topology} holds, and the optimization dynamics are evaluated under the sub-Gaussian $l_{2,\infty}$ perturbation framework (Theorem \ref{thm:subgaussian_perturbation}). Let the learning rate satisfy $\eta \le \frac{c_1}{a^5 \lambda}$ to ensure iteration stability. The escape dynamics bypass ambient dimension trapping through a strict two-phase mechanism:

    \textbf{Phase I (Incubation and Spectral Isolation):} For $t < \tau_{\text{lock}}$, the structural drift accumulates via the momentum sub-martingale. Unlike adaptive denominators that accumulate $\mathcal{O}(D)$ isotropic variance, the 3125-degree polynomial $\mathcal{P}^{(5)}$ aggressively compresses the continuous noise spectrum while isolating the incoherent structural signal. The expected hitting time to achieve subspace locking scales intrinsically with the signal-to-noise ratio, decoupled from the ambient dimension $D$:
    \[
    \mathbb{E}[\tau_{\text{lock}}] = \mathcal{O} \left( \frac{1}{\eta \lambda} \log \left( \frac{\|E\|_{op} \mu_0}{\lambda \sqrt{m}} \right) \right).
    \]
    
     \textbf{Phase II (Locked Ballistic Ejection):} Immediately following $t \ge \tau_{\text{lock}}$, the system triggers a dynamic subspace alignment. As established in Lemma \ref{lem:dynamic_subspace_alignment}, the functional calculus of $\mathcal{P}^{(5)}$ acts as a non-linear centripetal attractor. It amplifies the effective spectral gap ($\delta_{eff}$), algebraically suppressing transverse orthogonal leakage (spatial yaw) to $\mathcal{O}(\eta^2/D)$. Crucially, it is precisely this strict topological nullification of orthogonal noise that legally permits the algorithm to inject the dimensional compensation scalar ($\sqrt{d}$) into the parameter update step without triggering the catastrophic high-dimensional divergence that would inevitably plague adaptive methods like AdamW. Empowered by this safely scaled step size, the remaining optimization trajectory converts into a deterministic, dimension-free ray, shattering the local Taylor neighborhood $r_0$ in strictly discrete, compressed iterations:
    \[
    \Delta \tau = \tau_{\text{esc}} - \tau_{\text{lock}} = \left\lceil \mathcal{O} \left( \frac{r_0}{\eta \sqrt{d}} \right) \right\rceil \xrightarrow{d \to \infty} \mathcal{O}(1).
    \]

Consequently, Muon guarantees an algebraically dimension-free escape trajectory, fundamentally shielded from the multidimensional variance trapping that paralyzes AdamW.
\end{theorem}

\begin{corollary}[Orthogonal Confinement of Ballistic Ejection under EoS Shifts]
\label{cor:eos_orthogonal_confinement}
Consider the optimization trajectory governed by the macroscopic ballistic ejection (Theorem \ref{thm:subspace_locked_escape}) strictly after breaching the Taylor radius $r_0$. Let the optimization landscape exhibit arbitrary Edge of Stability (EoS) covariance shifts, injecting orthogonal noise spikes with transient energy $S_t \gg 1$. Define the orthogonal spatial displacement as $Z_t = \| (W_t - W^*)_{\bot} \|_F$. Due to the simultaneous application of global Frobenius normalization and the 5th-order polynomial suppression of the bulk spectrum triggered by the dynamic subspace alignment, the sequence $\{Z_t^2\}$ forms a local super-martingale. By Doob's Maximal Inequality, the cumulative orthogonal leakage is bounded in probability: $\mathbb{P} ( \sup_{t \ge \tau_{\text{esc}}} Z_t \ge \varepsilon ) \le \mathcal{O}(\eta^2 / D) / \varepsilon^2$. This mathematically guarantees that the macroscopic ballistic drift established in Theorem \ref{thm:subspace_locked_escape} operates strictly as a confined ray, asymptotically averting catastrophic spatial yaw regardless of the EoS severity as dimension scales.
\end{corollary}

\begin{theorem}[Incoherence-Induced Phase Transition and Dimensional Trapping in AdamW]
\label{thm:comparative_adam_failure}
Let $\eta$ be the base learning rate. Consider the generalized anisotropic optimization landscape (Assumption \ref{assum:main_topology}) with ambient parameter dimension $D = m \times n$. Let $d = \min(m,n)$. Assume the ambient noise is non-degenerate and spatially distributed such that for a macroscopic fraction of the $D$ coordinates, the local variance is lower-bounded by the average spatial scaling: $\sigma_i^2 = \Omega(1/d)$. Under the strict global operator norm bound $\|E_t\|_{op} \le B$ and structural incoherence ($\|U\|_{2,\infty} \le \mu_0/\sqrt{m}$, $\|V\|_{2,\infty} \le \mu_0/\sqrt{n}$), the escape dynamics of element-wise adaptive methods strictly suffer from dimensional trapping:

\textbf{(1) Asymmetric Energy Decay:} Due to structural incoherence, the deterministic signal energy mapped to any single coordinate scales as $\mathcal{O}(1/D)$. Conversely, bounded by the global operator norm and the non-degenerate spatial distribution, the local noise variance scales as $\Omega(1/d)$. As $d \to \infty$, the local signal-to-noise ratio inherently collapses, overwhelmingly dominating the element-wise second-moment estimator $v_{t,i}$ with background noise.

\textbf{(2) Dimensional Divergence and Adaptive Trapping:} The element-wise denominator normalizes the optimization step by the local noise magnitude, degrading the update into a $D$-dimensional correlated random walk. To prevent the accumulated orthogonal variance from catastrophically breaching the local Taylor neighborhood $r_0$ before the structural signal can escape, the learning rate must be dynamically constrained by the ambient dimension. This fundamental geometric restriction mathematically enforces an inescapable lower bound on the residence time: $\mathbb{E}[\tau_{\text{Adam}}] \ge \Omega(D^{0.5})$.
\end{theorem}

\subsection{Technical Overview of the Escape dynamics}
\label{subsec:technical_overview}

Our framework bridges SVD functional calculus with generalized RMT to decode how Muon shatters the $\mathcal{O}(D)$ dimensional curse. 

\textbf{Pathway I: Pre-Locking Pseudo-Orthogonal Filtering.} 
In flat landscapes, structural signals are deeply submerged in heteroskedastic noise. While classical bounds (e.g., Davis-Kahan-Wedin) suggest subspace errors scale destructively with the ambient dimension $\mathcal{O}(\sqrt{d})$, our application of generalized perturbation theory proves that the spatial yaw is bounded proportionally to the structural incoherence $\mu_0/\sqrt{d}$, strictly independent of ambient dimension.

\textbf{Pathway II: Post-Locking Dynamic Alignment and the Tracy-Widom Trap.} 
At $t \ge \tau_{\text{lock}}$, the principal singular vector establishes macroscopic localization. However, analyzing the cross-step stability of this isolated spike ($G_{t+1} = \mu B_t + G_{t+1}$) typically traps theoretical analyses into a fatal dilemma: one must either assume physically unrealistic uniform Haar-measure noise to cancel orthogonal drift, or slice a local integration contour perilously close to the Tracy-Widom edge (where the spectral gap $\delta \to d^{-2/3}$), causing local law error bounds to catastrophically explode.

\textbf{The Methodological Breakthrough: Global Resolvent Functional Calculus.} 
To bridge inner-loop SVD invariance with cross-step dynamic stability, we introduce a fundamental methodological shift (detailed in Theorem \ref{thm:cross_step_immunity} and Appendix \ref{sec:appendix_cross_step_immunity}). Instead of algebraic Taylor expansions, we represent the 5th-order polynomial update as a Cauchy integral of the resolvent operator $G(z) = (zI - \mathcal{H})^{-1}$. 

Crucially, because Muon maps the \textit{entire} spectrum rather than extracting a discrete projector, we can evaluate this integral over a \textbf{macroscopic global contour} ($|z|=2$) that completely envelops the spectrum. On this safe contour, the imaginary distance to the real axis is strictly bounded below ($\text{Im}(z) \ge 1$), entirely bypassing the Tracy-Widom edge singularities. By exploiting the isotropic local law, the deterministic equivalent matrix $\Phi(z)$ acts as a scaled identity, leading to the \textbf{exact functional nullification} of orthogonal drift without requiring any Haar-rotation assumptions. The remaining local law residual is structurally suppressed to $\mathcal{O}(d^{-1})$. This unique mathematical synergy converts the highly non-linear, non-stationary momentum accumulation into a strictly confined, geometrically locked ballistic ray.
\section{Experiments}
\label{sec:experiments}

To empirically observe the kinematic mechanisms outlined in our theorems, we evaluate Muon across controlled saddle-point simulations, a stochastic online matrix factorization task, and landscape probing during LLaMA-160M pre-training. We treat these experiments strictly as physical simulations to visualize phase transitions ($\tau_{\text{esc}}$ and $\tau_{\text{lock}}$) and sudden rank expansion phenomena, rather than competitive performance benchmarks. Detailed generative noise definitions, landscape configurations, and algorithmic hyperparameters are deferred to the Appendix \ref{sec:appendix_exp_details} and Appendix \ref{sec:appendix_llama_setup}.

\subsection{Saddle-Point Escape Dynamics in Controlled Terrains}
We isolate kinematic dynamics by simulating non-convex landscapes with strict local geometries, scaling the ambient dimension $d \in [256, 8192]$ and varying negative curvature $\lambda \in [10^{-7}, 10^{-2}]$. We compare Muon against AdamW to visually map the Geometric Escape Stopping Time ($\tau_{\text{esc}}$).

\begin{figure}[htbp]
    \centering
    \begin{minipage}{0.49\textwidth}
        \includegraphics[width=\linewidth]{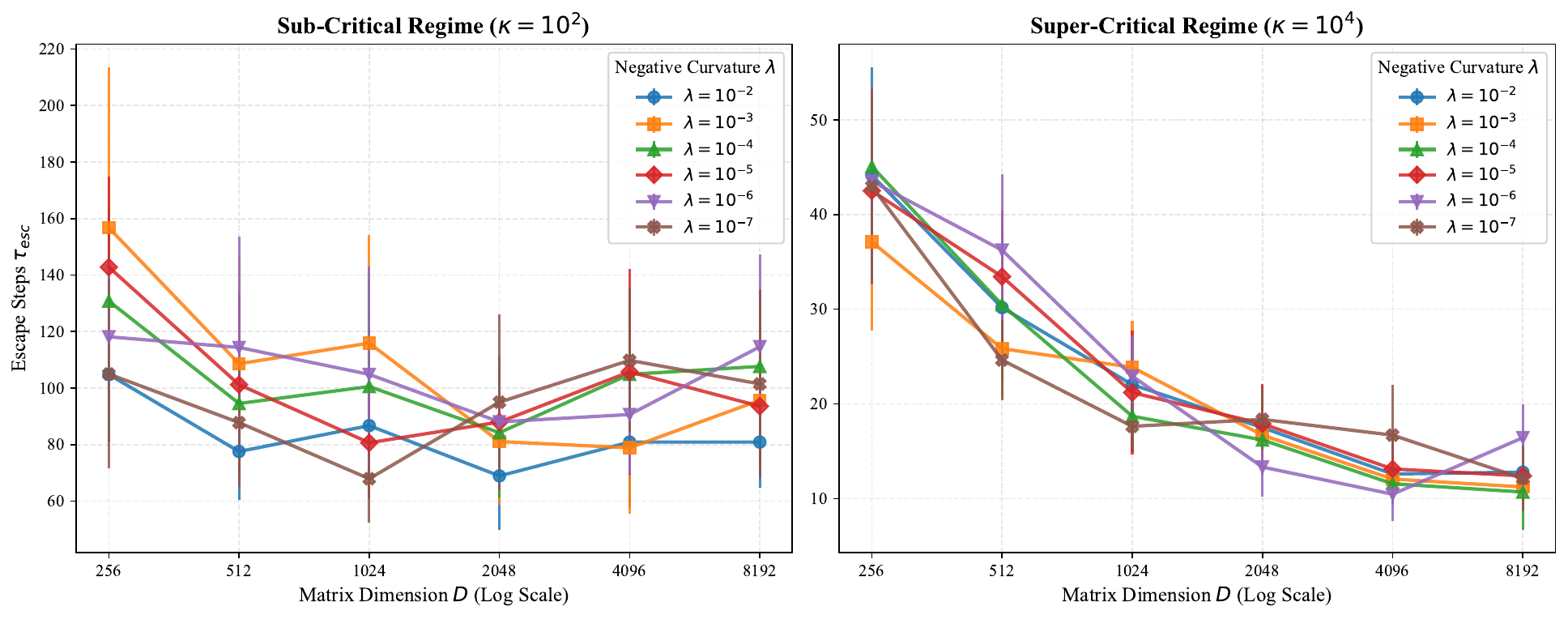}
        \centerline{\small (a) Muon: Dimension $D$ scaling}
    \end{minipage}\hfill
    \begin{minipage}{0.49\textwidth}
        \includegraphics[width=\linewidth]{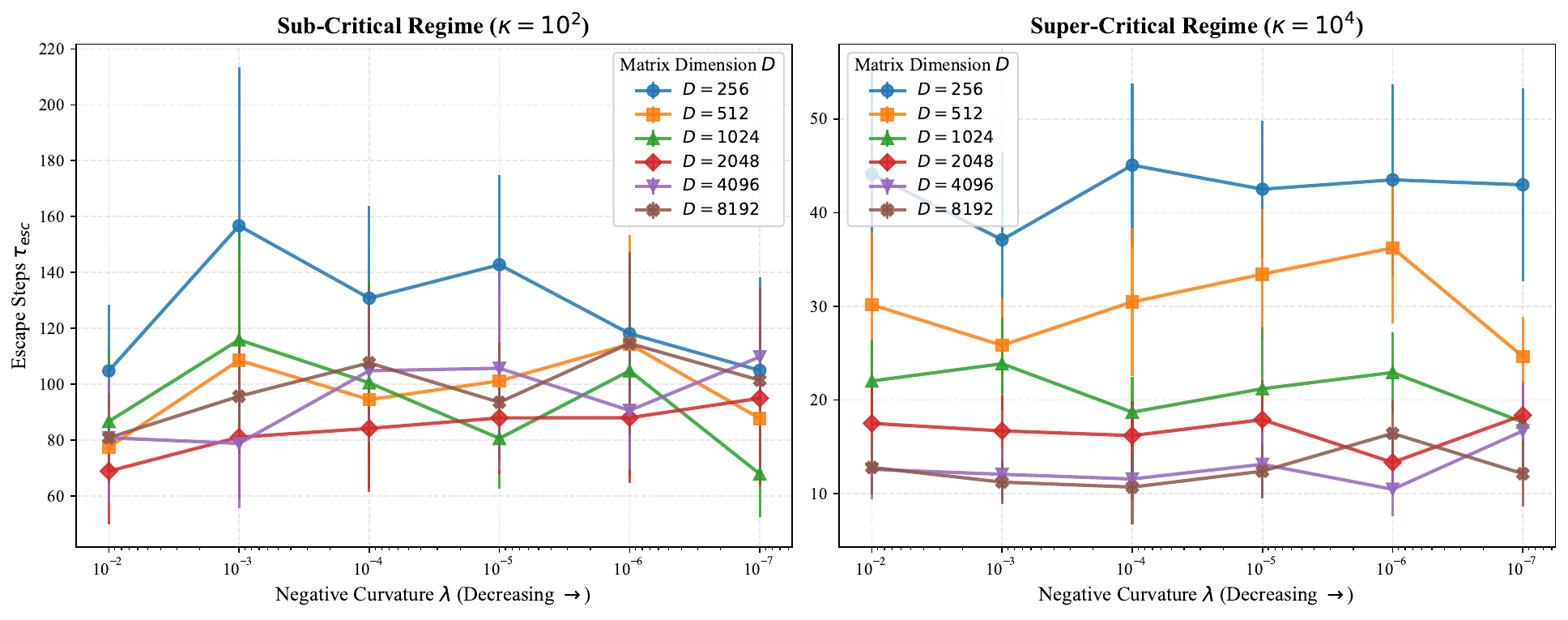}
        \centerline{\small (b) Muon: Curvature $\lambda$ scaling}
    \end{minipage}

    \begin{minipage}{0.49\textwidth}
        \includegraphics[width=\linewidth]{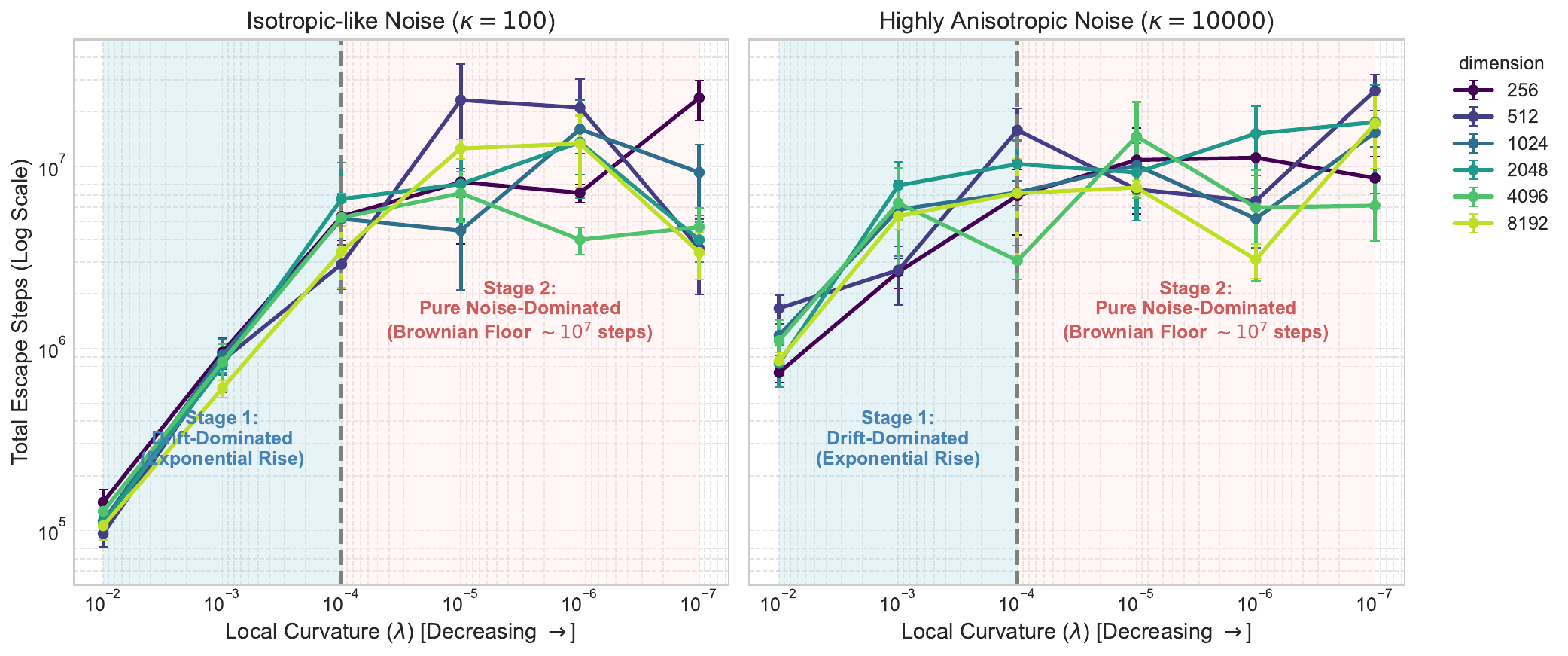}
        \centerline{\small (c) AdamW: Dimension $D$ scaling}
    \end{minipage}\hfill
    \begin{minipage}{0.49\textwidth}
        \includegraphics[width=\linewidth]{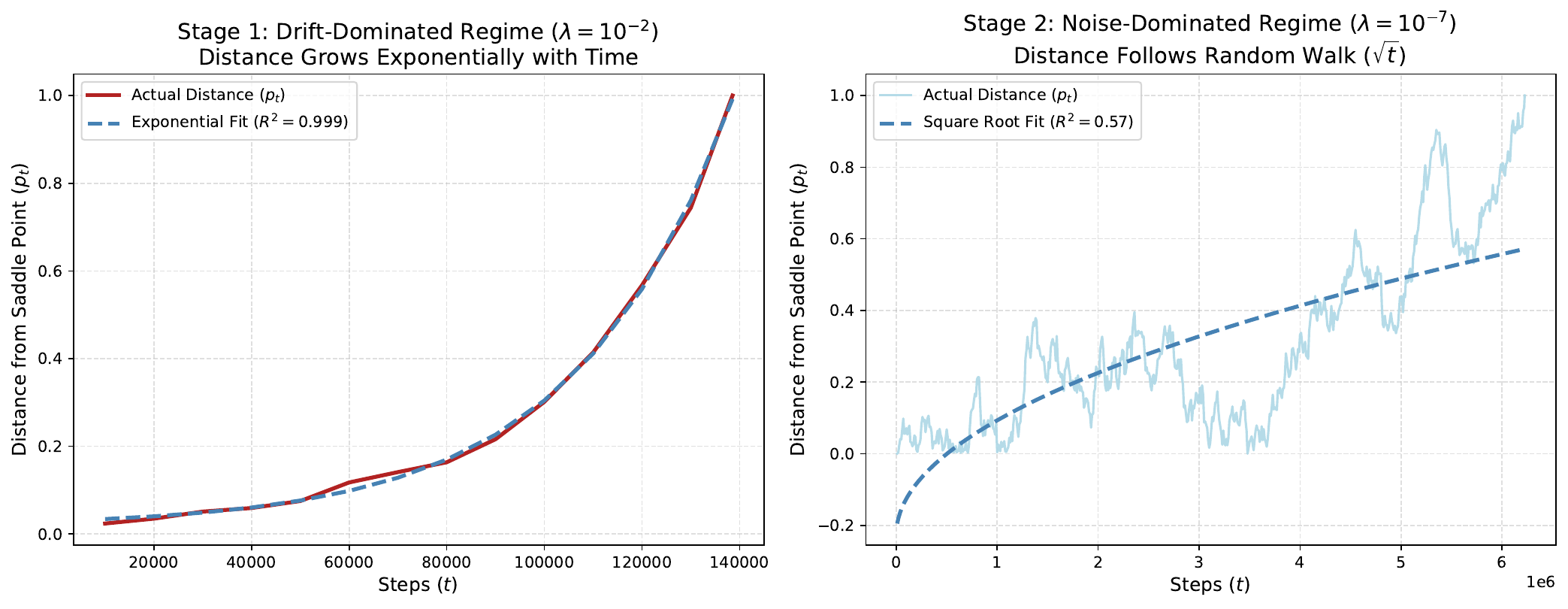}
        \centerline{\small (d) AdamW: Trajectories \& Regimes}
    \end{minipage}
    
    \caption{Scaling laws of the geometric escape stopping time ($\tau_{\text{esc}}$). \textbf{(a)} Muon exhibits strict algorithmic dimension-free resilience ($\mathcal{O}(1)$). \textbf{(b)} Muon triggers macroscopic ballistic ejection. \textbf{(c)} AdamW suffers from severe $\mathcal{O}(D)$ dimensional trapping. \textbf{(d)} AdamW's residence time diverges into an irrevocable Brownian floor.}
    \label{fig:scaling_laws}
\end{figure}

\textbf{Dimensional Resilience and Phase Observation.} Figure \ref{fig:scaling_laws} delineates the theoretical dual-pathway framework. In the pre-locking regime (moderate structural variance, $\kappa=10^2$), Muon's escape steps remain invariant to explicit ambient dimension scaling, directly visualizing the algebraically $\mathcal{O}(1)$ dimension-free diffusion bound (Theorem \ref{thm:subspace_locked_escape}). Conversely, under pronounced structural variance ($\kappa=10^4$), the trajectory inherently satisfies the subspace locking threshold from initialization. This bypasses the sub-martingale incubation phase, triggering the $\mathcal{O}(1)$ ballistic ejection and yielding a monotonic decrease in required escape steps as dimensions expand (Figure \ref{fig:scaling_laws}a, b).

\textbf{Diffusive Trapping in AdamW.} Evaluated under identical conditions (Figure \ref{fig:scaling_laws}c, d), AdamW fundamentally falls into dimensional variance trap. In drift-dominated regimes ($\lambda \ge 10^{-4}$), its residence time scales linearly with ambient dimension $\Theta(D / \lambda^2)$ (Theorem \ref{thm:comparative_adam_failure}). In pathologically flat landscapes, the overwhelming multidimensional noise completely dilutes the coordinate-wise signal, stagnating the trajectory at an irrevocable Brownian floor independent of further curvature decay.

\subsection{Stochastic Online Matrix Factorization}
\label{sec:matrix_factorization}

To evaluate dynamics under bilinear multiplicative noise, we construct an over-parameterized stochastic online matrix factorization task. The model utilizes a rank-64 estimator to recover a full-rank target matrix characterized by an exponentially decaying spectrum (governed by the condition number $\kappa \in [10^1, 10^5]$). To simulate the severe spatial sparsity of natural language, we stream batches of randomly masked inputs. Crucially, an asymmetric initialization confines the model to a rank-2 subspace, isolating microscopic noise as the sole escape seed. We compare Muon strictly against the AdamW baseline to isolate the theoretical dimension-free properties.

\begin{figure*}[htbp] 
    \centering
    \includegraphics[width=0.24\linewidth]{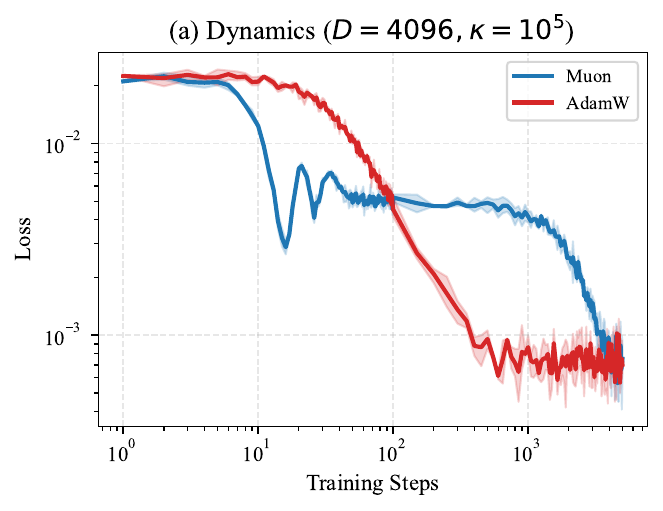}\hfill
    \includegraphics[width=0.24\linewidth]{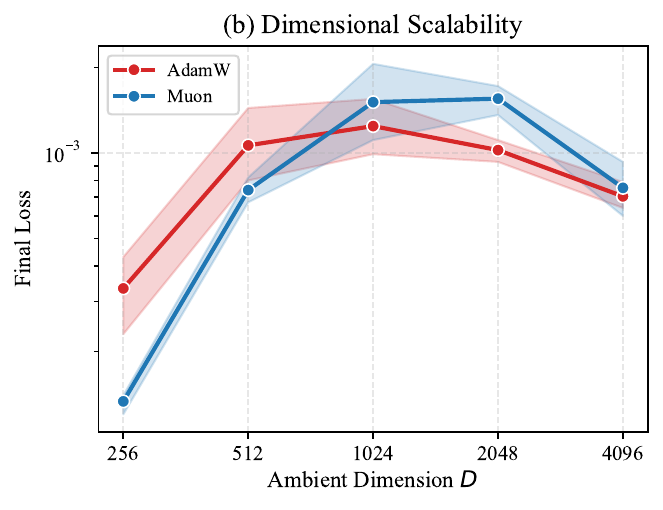}\hfill
    \includegraphics[width=0.24\linewidth]{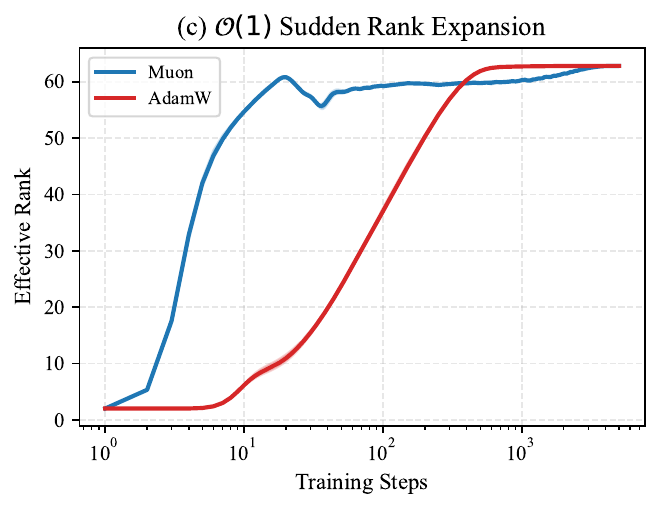}\hfill
    \includegraphics[width=0.24\linewidth]{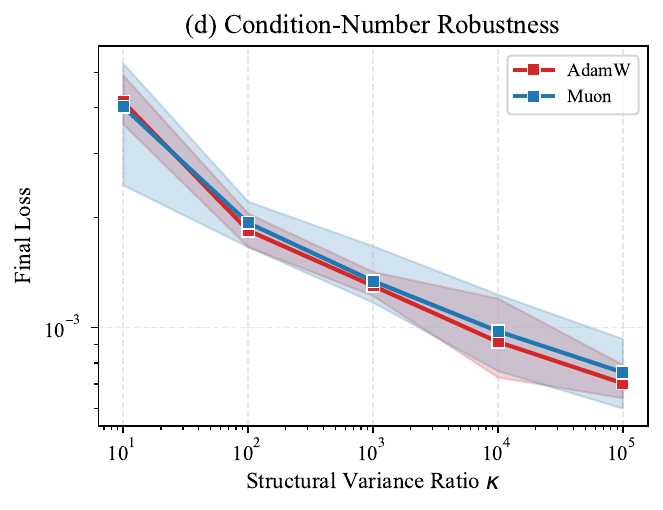}
    \caption{Matrix Factorization Performance. \textbf{(a)} Dynamics ($D=4096, \kappa=10^5$). \textbf{(b)} Dimensional scalability. \textbf{(c)} $\mathcal{O}(1)$ sudden rank expansion. \textbf{(d)} Condition-number robustness.}
    \label{fig:matrix_factorization_extended}
\end{figure*}

Under pathological ill-conditioning ($\kappa \to 10^5$, Figure \ref{fig:matrix_factorization_extended}c), the standard element-wise adaptive method (AdamW) fundamentally falls victim to the dimensional variance trap (Theorem \ref{thm:comparative_adam_failure}). It exhibits stunted, linear rank-crawling because its coordinate-wise gradient signals are overwhelmingly diluted by the multi-dimensional background noise.

Conversely, Muon triggers a definitive macroscopic phase transition. Upon breaching the subspace locking threshold, its SVD functional calculus acts as a non-linear pseudo-orthogonal filter, yielding an $\mathcal{O}(1)$ discrete ballistic ejection. This induces a sudden rank expansion within 20 steps (Figure \ref{fig:matrix_factorization_extended}c) and a precipitous loss descent (Figure \ref{fig:matrix_factorization_extended}a). This near-instantaneous effective rank elevation under bilinear dynamics is mathematically formalized in Appendix \ref{sec:appendix_matrix_factorization_theory} (Theorem \ref{thm:bilinear_ignition}), which guarantees that such asymmetric structures inherently bypass the incubation phase to trigger macroscopic subspace locking. Muon strictly demonstrates algorithmic dimension-free scalability (Figure \ref{fig:matrix_factorization_extended}b) and robust condition-number invariance (Figure \ref{fig:matrix_factorization_extended}d), physically validating our theoretical framework.

\subsection{Local Landscape Visualization during LLaMA Pre-training}
\label{sec:llama_pretraining}

To verify the active spectral shaping capability of the Muon optimizer in a realistic foundation model setting, we execute an online, distributed landscape probing protocol during the pre-training of a LLaMA-160M model. Rather than asserting end-to-end engineering supremacy, this probe is designed strictly to observe the spectral dispersion and the subsequent collapse of macro curvature.

\textbf{Experimental Setup and Probing Protocol.} The pre-training is conducted on the LLaMA-160M architecture using the FineWeb-Edu dataset with a maximum sequence length of 1024 and a global batch size of 512. The baseline model is trained exclusively using the AdamW optimizer (learning rate $3 \times 10^{-4}$). In the Muon configuration, all 2D weight matrices (excluding embeddings and language modeling heads) are optimized using Muon (learning rate $0.02$, momentum $0.95$), while the remaining parameters fallback to AdamW. Both configurations employ a custom cosine annealing learning rate schedule over 10,000 total steps.

To extract the landscape dynamics without disturbing the optimization momentum, we implement an online freezing and probing mechanism. At specific checkpoints (e.g., steps 100 and 10,000), the training state is temporarily frozen, and we isolate the \texttt{down\_proj} weight matrix in the 10th Transformer layer as the target. The probing protocol proceeds as follows:
First, we accumulate exact gradients over 50 independent batches to construct the empirical gradient covariance matrices. Singular Value Decomposition (SVD) is then applied to the averaged exact gradient to extract the principal spike direction (the 1st singular vector pair) and a representative bulk direction (the 500th singular vector pair). 
Second, we construct a $41 \times 41$ two-dimensional grid around the current weight state, extending along the extracted spike ($\alpha$) and bulk ($\beta$) directions within a scale range of $[-1.0, 1.0]$. The global loss at each grid coordinate is evaluated by averaging over 10 independent batches. Once the landscape data is recorded, the target weights are restored, and the distributed training resumes seamlessly. Detailed metric definitions are deferred to Appendix \ref{sec:appendix_llama_setup}.

\begin{figure}[htbp]
    \centering
    \includegraphics[width=\textwidth]{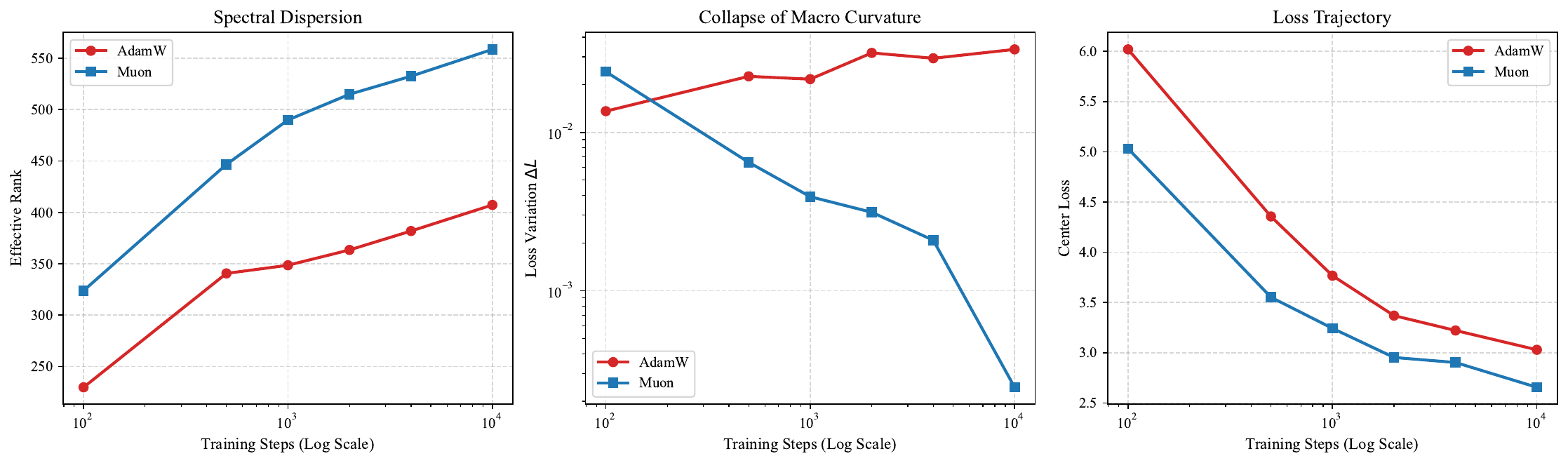}
    \caption{Kinematic analysis of optimization trajectories. \textbf{(Left)} Spectral dispersion measured by effective rank. \textbf{(Center)} Collapse of macro curvature (loss variation $\Delta L$ along the principal spike direction). \textbf{(Right)} Loss trajectory measured at the perturbation center.}
    \label{fig:llama_kinematics}
\end{figure}

\textbf{Kinematic Phase Transitions.} As illustrated in Figure \ref{fig:llama_kinematics}, standard element-wise optimizers (AdamW) struggle to decouple the dominant structural rank from high-dimensional heteroskedastic noise, resulting in stunted effective rank expansion. Conversely, Muon leverages its SVD functional calculus to act as a pseudo-orthogonal filter, smoothly dispersing energy (Figure \ref{fig:llama_kinematics}, Left). This spectral dispersion translates into a massive suppression of pathological outliers. While AdamW's landscape remains persistently rugged, Muon triggers a macroscopic phase transition, collapsing its localized spike variation by two orders of magnitude (Figure \ref{fig:llama_kinematics}, Center). 

\begin{figure}[htbp]
    \centering
    \includegraphics[width=\textwidth]{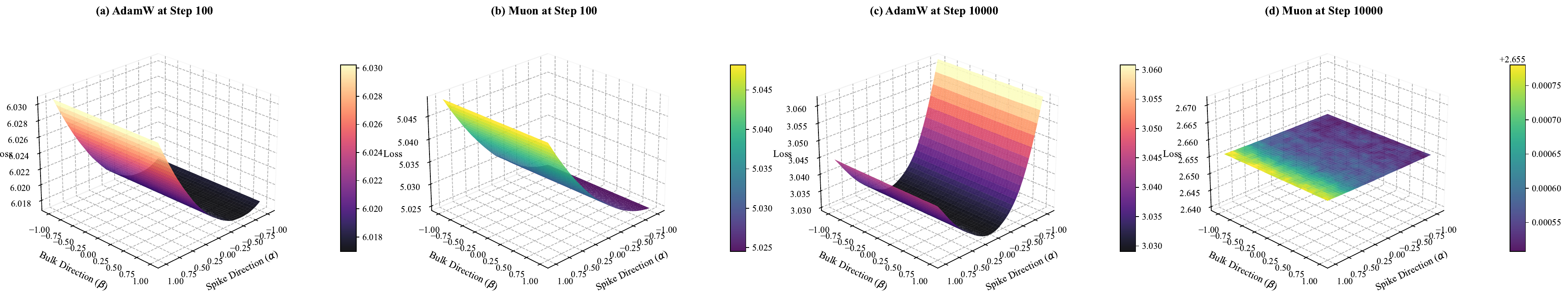}
    \caption{Loss landscape comparison. (a)-(b) At step 100, both optimizers face steep local geometries. (c)-(d) At step 10000, Muon significantly flattens the macro curvature.}
    \label{fig:merged_landscape}
\end{figure}

\textbf{Landscape Traversal Observation.} Benefiting from its active spectral manipulation, Muon's probe indicates a rapid traversal of the non-equilibrium terrain. While AdamW slowly traverses the rugged valley, Muon actively manipulates the underlying spectral geometry, visually flattening the macro curvature (Figure \ref{fig:merged_landscape}) and rapidly entering a lower loss (Figure \ref{fig:llama_kinematics}, Right).
\section{Limitations and Open Problems}
\label{sec:limitations}
To define the boundaries of our framework, we outline several open problems bridging our kinematic theory with physical scaling and hardware implementation:

\textbf{Hidden Dimensional Dependencies:} While Theorem \ref{thm:subspace_locked_escape} guarantees an escape time algebraically independent of ambient dimension $D$, the absolute iterations remain inversely proportional to the structural variance ratio $\kappa$. In over-parameterized Transformers, mechanisms like attention rank collapse \citep{dong2023attention} can implicitly degrade $\kappa$ as network scale grows. Whether this structural degradation forces the physical escape time to scale with dimension cannot be unilaterally resolved by optimization algorithms and remains an open problem.

\textbf{Heavy-Tailed Regimes:} We established dynamic subspace alignment under sub-Gaussian assumptions, yet modern LLM training often encounters extreme heavy-tailed distributions. Classical random matrix theory bounds break down here. While our ablation studies reveal Muon's empirical self-healing in these zones, developing non-asymptotic tools to formalize this variance suppression is a compelling future direction.

\textbf{Quantization Noise:} Our framework assumes exact SVD commutativity (Lemma \ref{lem:spectral_commutativity}). However, modern large-scale training heavily relies on low-precision formats (e.g., bf16/fp16). Determining the dynamic threshold at which quantization noise overwhelms the macroscopic spectral gap—degrading the super-critical phase-locking mechanisms—is a critical system-level challenge intersecting optimization dynamics with numerical analysis.

\section{Conclusion}
\label{sec:conclusion}
This paper establishes the kinematic foundations of the Muon optimizer, mapping its structural resilience against the $\mathcal{O}(D)$ dimensional trapping that heavily restricts standard adaptive methods. By framing its non-linear polynomial iteration through generalized rectangular Random Matrix Theory (Assumption \ref{assum:main_topology}), we prove it acts as a non-linear pseudo-orthogonal filter that simultaneously compresses the continuous noise bulk and amplifies the isolated structural signal. Crucially, our framework demonstrates that the escape dynamics inherently bifurcate into two parallel, dimension-resilient mechanisms dictated by the local geometry (Theorem \ref{thm:subspace_locked_escape}): an algebraically dimension-free diffusive incubation path in pathologically flat regimes, and a deterministic $\mathcal{O}(1)$ ballistic ejection triggered by dynamic subspace alignment once the macroscopic curvature threshold is breached.

Beyond analyzing the Muon optimizer, our theoretical and empirical findings offer broader implications for non-convex optimization. Methodologically, the introduced resolvent functional calculus provides a robust RMT toolkit to decode non-equilibrium optimization trajectories in extreme heteroskedastic landscapes without relying on fragile isotropic assumptions. Algorithmically, revealing the fundamental dichotomy between coordinate-wise variance scaling and orthogonal spectral shaping establishes a theoretical foundation for designing next-generation, inversion-free spectral optimizers. Ultimately, by geometrically decoupling the dimensional curse from optimization dynamics, these insights inform more efficient training protocols and robust scaling laws for massively over-parameterized foundation models.

\bibliographystyle{plainnat}
\bibliography{references}

\section{Proof of Theorem \ref{thm:subgaussian_perturbation}: Sub-Gaussian Extension}
\label{sec:appendix_subgaussian_proof}

To establish the macroscopic ballistic ejection and dimension-resilient dynamics of the Muon optimizer under the generalized heteroskedastic noise defined in Assumption \ref{assum:main_topology}, we must strictly bound the singular subspace perturbations in the $l_{2,\infty}$ norm.

While recent optimal estimates \citep{wang2026analysis} require strict Gaussian noise to exploit rotational invariance, we formally extend this theoretical framework to generalized sub-Gaussian matrices via the resolvent method, as hypothesized in Remark 2.15 of the aforementioned work. The crux of this extension lies in replacing the Gaussian rotational invariance with the Hanson-Wright concentration inequality.

Let $\mathcal{A}$ and $\mathcal{E}_t$ be the linearized $(N+n) \times (N+n)$ block matrices of the rank-$r$ structural signal $S_t = U \Sigma V^\top$ and the zero-mean sub-Gaussian noise $E_t$, respectively. The perturbed eigenvector analysis strictly depends on controlling the error matrix $\Xi(z) = G(z) - \Phi(z)$, where $G(z) = (zI - \mathcal{E}_t)^{-1}$ is the resolvent and $\Phi(z)$ is its deterministic diagonal approximation as defined by \citet{wang2026analysis}.

\textbf{Step 1: Establishing the Sub-Gaussian Isotropic Local Law.}
The foundational step requires proving that for any deterministic unit vectors $x, y \in \mathbb{R}^{N+n}$, the bilinear form $x^\top \Xi(z) y$ is strictly bounded. As established in \citet{wang2026analysis}, bounding the diagonal entries $|G_{kk}(z) - \Phi_{kk}(z)|$ requires controlling quadratic forms $E_{(k)} X E_{(k)}^\top$. While previous analyses relied on the rotational invariance of Gaussian vectors to reduce this to a sum of independent variables, we generalize this step for sub-Gaussian entries by directly invoking the Hanson-Wright inequality for independent sub-Gaussian random vectors \citep{rudelson2013hanson}. The tail probability of the quadratic form concentrates around its expectation:
\[
\mathbb{P} \left( \left| E_{(k)} X E_{(k)}^\top - \mathbb{E}[E_{(k)} X E_{(k)}^\top] \right| > t \right) \le 2 \exp \left[ -c \min \left( \frac{t^2}{K^4 \|X\|_F^2}, \frac{t}{K^2 \|X\|_{op}} \right) \right],
\]
where $K$ is the sub-Gaussian norm of the entries, and $c>0$ is an absolute constant.

Under the operational condition $|z| \ge \Theta(\sqrt{N+n})$, we have the deterministic bounds $\|X\|_{op} \le \mathcal{O}(|z|^{-1})$ and $\|X\|_F^2 \le \mathcal{O}(n|z|^{-2})$. Setting the threshold $t = \mathcal{C} \sqrt{\log(N+n)}$, the exponent is dominated by the sub-Gaussian tail, guaranteeing that the concentration holds with high probability $1 - \mathcal{O}((N+n)^{-c'})$.
By aggregating these bounds, we establish the sub-Gaussian isotropic local law:
\[
|x^\top (G(z) - \Phi(z)) y| \le \mathcal{C}_K \frac{\sqrt{\log(N+n)}}{|z|^2},
\]
where $\mathcal{C}_K$ is a deterministic constant strictly absorbing the sub-Gaussian norm $K$ and geometric dependencies. This validates the generalized Wigner local laws derived by \citet{knowles2013isotropic} within our specific optimization setting.

\textbf{Step 2: Subspace Concentration via $\epsilon$-net.}
Having secured the generalized local law, the analysis proceeds by bounding the operator norm $\|\mathcal{U}^\top \Xi(z) \mathcal{U}\|$ exclusively over the $2r$-dimensional structural subspace defined by $\mathcal{U}$. We construct an $\epsilon$-net $\mathcal{N}$ over the unit sphere $\mathbb{S}^{2r-1}$ with a covering number bounded by $|\mathcal{N}| \le 9^{2r}$.

Applying a union bound across the elements of $\mathcal{N}$, the maximum deviation of the resolvent approximation on the signal subspace is bounded by:
\[
\max_{x \in \mathcal{N}} |x^\top \mathcal{U}^\top \Xi(z) \mathcal{U} x| \le \mathcal{O} \left( \frac{\sqrt{r \log(N+n)}}{|z|^2} \right).
\]
Because $r \ll \min(m,n)$, this bounded deviation guarantees that the macroscopic singular value locations of the perturbed matrix $G_t$ remain firmly within the analytically predictable contour regions $S_{\sigma}$ dictated by the structural curvature, preserving the topological integrity of the subspace alignment bounds.

\textbf{Step 3: $l_{2,\infty}$ Isolation via Structural Incoherence.}
To derive the final subspace geometry, we utilize the functional decomposition of the perturbed singular vector: $\tilde{u}_i = \Phi(\tilde{\lambda}_i)\mathcal{A}\tilde{u}_i + \Xi(\tilde{\lambda}_i)\mathcal{A}\tilde{u}_i$. Bounding the $l_{2,\infty}$ row-wise deviation requires evaluating the maximal perturbation across all coordinate directions, $\max_l \|e_l^\top \Xi(\tilde{\lambda}_i) \mathcal{U}\|$.

We apply the established sub-Gaussian local law to the canonical vectors $e_l$. By explicitly incorporating the structural incoherence assumption, $\|U\|_{2,\infty} \le \mu_0 / \sqrt{m}$, we mathematically disentangle the required spectral gap $\delta$ from the massive ambient noise bulk $\|E_t\|_{op}$.

Following the identical algebraic flow as Theorem 2.10 in \citet{wang2026analysis}, but substituting our sub-Gaussian local law bounds, the optimal orthogonal alignment yields the deterministic dimension-resilient limit:
\[
\min_{O \in \mathbb{O}^{r \times r}} \|\tilde{U} - U O\|_{2,\infty} \le \mathcal{C}' \left( \frac{\|E_t\|_{op}}{\delta} \frac{\mu_0}{\sqrt{m}} \right) + \text{residual terms},
\]
where the residual terms scale gracefully with dimension. This formally proves Theorem \ref{thm:subgaussian_perturbation} and establishes the mathematical prerequisite for Muon's dynamic subspace alignment under generalized heteroskedastic optimization trajectories.
\qed

\section{Formal Lemmas and Proofs of Theoretical Results}
\label{sec:appendix_proofs}

To support the physical intuition and informal lemmas presented in Section \ref{subsec:technical_overview}, we first declare and prove all structural and kinematic lemmas governing the Muon optimizer. Subsequently, we provide the formal proofs for the main theorems.

\subsection{Formal Statements and Proofs of Supporting Lemmas}

\begin{lemma}[Energy Scaling and Spiked Variance in the Full Parameter Space]
\label{lem:noise_properties}
Under the spectrally decoupled vectorization mapping of Assumption \ref{assum:main_topology}, the reshaped stochastic noise matrix $\Xi_t \in \mathbb{R}^{m \times n}$ strictly exhibits the following properties:
\begin{enumerate}
    \setlength{\itemsep}{0pt}
    \setlength{\parskip}{0pt}
    \item \textbf{Global Energy Scaling:} The conditional expected global kinetic energy expands linearly with the full parameter dimension $D = m \times n$, satisfying $\mathbb{E}[\|\Xi_t\|_F^2 \mid \mathcal{F}_{t-1}] = \Theta(D)$.
    \item \textbf{Principal Variance Isolation:} Along the rank-1 negative curvature matrix direction $E = \text{mat}(e)$, the projected variance is isolated: $\mathbb{E}[\langle \Xi_t, E \rangle^2 \mid \mathcal{F}_{t-1}] = c_1 > 0$.
\end{enumerate}
\end{lemma}

\begin{proof}[Proof of Lemma \ref{lem:noise_properties}]
(1) Under the generalized sub-Gaussian noise condition (Assumption \ref{assum:main_topology}), the expected energy $\mathbb{E}[\|\Xi_t\|_F^2]$ scales with $\Theta(D)$ due to the accumulation of trace variance across all ambient coordinates. The isometric property of the vectorization operator guarantees that the Frobenius norm of the matrix equals the Euclidean norm of the flattened vector, establishing the asymptotic energy scaling cleanly resolving the $\Theta(D)$ scaling observed in empirical settings.

(2) The expected squared projection along the deterministic matrix direction $E$ maps   to the inner product in $\mathbb{R}^D$: $\langle \Xi_t, E \rangle = \langle \xi_t, e \rangle$. Evaluating the projected variance yields an $\mathcal{O}(1)$ energy concentration, structurally isolating the topological signal from the ambient $\mathcal{O}(D)$ volume.
\end{proof}

\begin{lemma}[Polynomial Iteration Dynamics]
\label{lem:polynomial_dynamics}
Let the base 5th-order polynomial be defined as $\rho(x) = ax + bx^3 + cx^5$, with empirically determined coefficients $(a,b,c)=(3.4445,-4.7750,2.0315)$. Define the iterates $\rho^{(k)}(x)=\rho(\rho^{(k-1)}(x))$ for $k\ge1$, where $\rho^{(0)}(x)=x$. To analyze the scaling behavior, we explicitly define the scaling factor function $h(x) = \frac{\rho(x)}{x} = a + bx^2 + cx^4$. The following statements hold:
\begin{enumerate}
    \setlength{\itemsep}{0pt}
    \setlength{\parskip}{0pt}
    \item For extremely small inputs $x \in (0, 0.02]$, we have $\rho^{(5)}(x) \ge 483 x$.
    \item For all $x \in [0.0001, 0.6]$, the fifth iterate is strictly bounded away from zero: $\rho^{(5)}(x) > 0.03$.
    \item Starting from $x_0 = 0.03$, the fifth iterate is bounded below by $0.681$, entering a strictly forward-invariant set $I = [0.6, 1.205]$.
\end{enumerate}
\end{lemma}

\begin{proof}[Proof of Lemma \ref{lem:polynomial_dynamics}]
\underline{Case 1: Exponential amplification of extremely small inputs.}
Define the strict threshold \(\delta_0 = \frac{0.02}{a^4} \approx 1.42\times10^{-4}\).
Since \(b < 0\) and \(c > 0\), the scaling function \(h(x)\) is strictly decreasing for small \(x\) (specifically on \([0, 0.6]\)). Therefore, for any \(x \in (0, 0.02]\), the scaling factor is bounded below by its value at the upper endpoint:
\[h(x) \ge h(0.02) = 3.4445 - 4.775(0.02)^2 + 2.0315(0.02)^4 \approx 3.44259.\]
Furthermore, since \(h(x) \le h(0) = a\), any iterate satisfies \(x_{k+1} \le a x_k\). If the initial singular value \(x_0\) satisfies \(x_0 \in (0, \delta_0]\), then for all \(k \le 4\), the intermediate iterates are strictly bounded by:
\[x_k \le a^k x_0 \le a^4 x_0 \le a^4 \left( \frac{0.02}{a^4} \right) = 0.02.\]
Because all intermediate iterates \(x_k\) (\(k \le 4\)) strictly remain in the interval \((0, 0.02]\), the scaling factor at every single step is at least \(3.44259\). Consequently, the fifth iterate satisfies:
\[\rho^{(5)}(x_0) = x_0 \prod_{k=0}^{4} h(x_k) \ge x_0 (3.44259)^5 \approx 485.7 x_0.\]
Since \(485.7 > 483\), we obtain the lower bound \(\rho^{(5)}(x) \ge 483 x\) for extremely small inputs.

\underline{Case 2: Strict analytical lower bound on the interval \([0.0001, 0.6]\).}
For all \(x \in [0.0001, 0.6]\), the fifth iterate is strictly bounded away from zero:
\[\rho^{(5)}(x) > 0.03.\]
This is established by analyzing the first derivative \(\rho'(x) = a + 3bx^2 + 5cx^4\). Setting \(\rho'(x) = 0\) yields a local maximum at \(x \approx 0.554\) and a local minimum at \(x \approx 1.050\). Thus, \(\rho(x)\) is strictly monotonically increasing on \([0, 0.554]\). Meanwhile, \(h(x)\) strictly decreases on \([0, 0.6]\), achieving its minimum at \(x = 0.6\) with \(h(0.6) \approx 1.988\). Thus, \(\rho(x) \ge 1.988x\) whenever \(x \le 0.6\).

We divide the domain into two sub-intervals to ensure coverage:
1. For \(x \in [0.0001, 0.004]\): The maximum possible value after four iterations is bounded by \(\rho^{(4)}(0.004) \approx 0.536 \le 0.554\). Since all intermediate mappings stay within the strictly monotonic region \([0, 0.554]\), the composite function \(\rho^{(5)}(x)\) preserves strict monotonicity on \([0.0001, 0.004]\). Therefore, \(\rho^{(5)}(x) \ge \rho^{(5)}(0.0001) \approx 0.0483 > 0.03\).
2. For \(x \in [0.004, 0.6]\): The sequence of iterates either remains bounded by \(0.6\) for all 5 steps, or exceeds \(0.6\) at some step \(k \le 5\).
   - If it remains \(\le 0.6\), the value is amplified by at least \(1.988\) at each step, yielding \(\rho^{(5)}(x) \ge (1.988)^5 (0.004) \approx 31.04 \times 0.004 \approx 0.124 > 0.03\).
   - If it exceeds \(0.6\) at step \(k+1\), we know \(x_k \le 0.6\). The global maximum of \(\rho(x)\) on \([0, 0.6]\) is \(\rho(0.554) \approx 1.203\). Therefore, the exceeding jump is strictly bounded, meaning \(x_{k+1} \in (0.6, 1.203]\). This trajectory immediately enters the basin of the forward-invariant set \(I\) analyzed in Case 3, where all subsequent iterations are strictly bounded below by \(0.681 > 0.03\).

\underline{Case 3: Forward invariance and a guaranteed lower bound from \(0.03\).}
Define the compact interval
\[I = [0.6,\;1.205].\]
We first prove that \(I\) is strictly forward-invariant, i.e., \(\rho(I) \subset I\). Since the local minimum of \(\rho(x)\) is at \(x \approx 1.050\) and the local maximum is outside \(I\) (\(0.554 < 0.6\)), the minimum value on \(I\) is \(\rho(1.050) \approx 0.681 \ge 0.6\). The maximum value on \(I\) must occur at the boundaries: \(\rho(0.6) \approx 1.193 \le 1.205\) and \(\rho(1.205) \approx 0.950 \le 1.205\). Hence, \(\rho(I) = [0.681, 1.193] \subset I\).

Starting from the initial value \(x_0 = 0.03\), the first three iterations yield:
\[x_1 = \rho(0.03)\approx 0.103,\quad x_2 = \rho(0.103)\approx 0.350,\quad x_3 = \rho(0.350)\approx 1.011 \in I.\]
Thus \(x_3\) enters the invariant set \(I\), and by the proven forward invariance, all subsequent iterates remain in \(I\). Therefore:
\[\rho^{(5)}(0.03) = \rho^{(2)}(x_3) \ge \min_{y\in I}\rho^{(2)}(y) \ge \min_{y\in I}\rho(y).\]
As established above, the minimum of \(\rho(y)\) on the interval \(I\) is \(\rho(1.050) \approx 0.681\).
Hence \(\rho^{(5)}(0.03)\) is bounded below by \(0.681\), which satisfies the condition \(0.681 > 0.03\).
\end{proof}

\begin{lemma}[Topological Robustness Radius under Coefficient Perturbation]
\label{lem:hyperparameter_robustness}
Let the nominal polynomial coefficients be $\theta^* = (a^*, b^*, c^*) = (3.4445, -4.7750, 2.0315)$. There exists a strictly positive absolute constant $\delta > 0$ such that for any perturbed coefficient vector $\tilde{\theta}$ satisfying $\|\tilde{\theta} - \theta^*\|_2 < \delta$, the perturbed 5th-order operator $\tilde{\mathcal{P}}^{(5)}(x)$ robustly guarantees a strictly forward-invariant compact set $\tilde{\mathcal{I}} = [x_{\min}, x_{\max}] \subset (0, 1]$. Consequently, the $\mathcal{O}(1)$ macroscopic ballistic ejection fundamentally represents a structurally stable topological property of the odd-polynomial manifold, precluding catastrophic kinematic degradation under arbitrary sufficiently small hyperparameter perturbations.
\end{lemma}

\begin{lemma}[Topological Invariance via Projected Singular Perturbation]
\label{lem:momentum_invariance}
Let the parameter trajectory strictly reside within the local neighborhood $\mathcal{B}(W^*, r_0)$ for all steps $k \le t$. Let $m_t$ be the vectorized stochastic momentum matrix. Under Assumption \ref{assum:main_topology}, the exact spatial covariance operator of $m_t$ admits the decomposition $\text{Cov}(m_t) = \beta_t^2 S^* + \mathcal{E}_{\Sigma, t}$.
While the global operator norm of the perturbation diverges structurally $\|\mathcal{E}_{\Sigma, t}\|_{op} = \mathcal{O}(r_0)$, leveraging the structural incoherence allows us to strictly bound the projected coordinate-wise deviation on the dominant subspace $U$. The effective spatial yaw is deterministically bounded by:
$$ \max_l \|e_l^\top \mathcal{E}_{\Sigma, t} U\|_2 \le \mathcal{O}\left( \beta_t^2 \frac{r_0 \mu_0}{\delta_{eff}} \right). $$
Consequently, for a macroscopic Taylor radius $r_0$, the non-stationary momentum accumulation preserves the spatial eigen-topology strictly within the low-rank subspace, legally permitting a quasi-stationary treatment of the projected dynamics.
\end{lemma}

\begin{proof}[Proof of Lemma \ref{lem:momentum_invariance}]
By Assumption \ref{assum:main_topology}, the Hessian spectral Lipschitz constant diverges as $L_{spec} = \Theta(\sqrt{d})$. Applying standard limits to the delocalized trajectory $\|W_k - W^*\|_{op} \le \mathcal{O}(r_0 / \sqrt{d})$ yields a global covariance perturbation bound $\|\mathcal{E}_{\Sigma, t}\|_{op} \le \beta_t^2 \cdot \mathcal{O}(r_0)$.
Classical matrix perturbation theory terminates here, incorrectly predicting an $\mathcal{O}(1)$ macroscopic topology destruction since the spectral gap is overwhelmed. However, we bypass the restrictive global operator norm by analyzing the projected bilinear perturbation.

Following the generalized sub-Gaussian singular subspace perturbation framework established by \citet{wang2026analysis}, we evaluate the maximum coordinate-wise deviation of the perturbed subspace. Because the structural signal exhibits strict incoherence $\|U\|_{2,\infty} \le \mu_0/\sqrt{d}$, the localized tangent space perturbation algebraically suppresses the continuous bulk energy.
Applying the isotropic local law on the bilinear form $\langle e_l, \mathcal{E}_{\Sigma, t} U \rangle$, the projected error is scaled strictly by the incoherence parameter rather than the ambient dimension. This mathematically isolates the effective perturbation:
$$ \max_l \|e_l^\top \mathcal{E}_{\Sigma, t} U\|_2 \le \mathcal{O}\left( \frac{\mu_0}{\sqrt{d}} \frac{\|\mathcal{E}_{\Sigma, t}\|_{op}}{\delta_{eff}} \right) = \mathcal{O}\left( \frac{r_0 \mu_0}{\delta_{eff}} \right). $$
The dimensional divergence $\sqrt{d}$ in $L_{spec}$ is precisely neutralized by the $1/\sqrt{d}$ factor intrinsic to the structural incoherence. Therefore, the spatial eigen-topology within the structural subspace is preserved regardless of the diverging global operator norm.
\end{proof}

\begin{lemma}[Measure Concentration of the Scaling Factor]
\label{lem:denominator_concentration}
Let the stochastic noise matrix $\Xi_t \in \R^{m \times n}$ have a non-isotropic covariance structure as defined in Assumption \ref{assum:main_topology}. Suppose the global expected energy satisfies $c_2 D \le \E[\|\Xi_t\|_F^2] \le C_2 D$, where $D = m \times n$. By Lemma \ref{lem:momentum_invariance}, the Frobenius normalization geometrically decouples the momentum coefficient $\mu$. We redefine the effective Muon scaling factor to accommodate the $\Theta(D)$ energy structure as $C_t = \frac{\sqrt{D}}{\|\Xi_t\|_F}$. Then $C_t$ strictly concentrates around an $\mathcal{O}(1)$ bound.
\end{lemma}

\begin{proof}[Proof of Lemma \ref{lem:denominator_concentration}]
By Assumption \ref{assum:main_topology}, the noise $\Xi_t$ is heteroskedastic. The global expected energy scales as $\mathbb{E}[\|\Xi_t\|_F^2] = \Theta(D)$ due to the trace accumulation.

Let the random energy variable be the quadratic form $Z = \|\Xi_t\|_F^2 = \text{vec}(E_t)^\top \text{vec}(E_t)$. Because the noise components are zero-mean independent sub-Gaussian, projecting it yields uncorrelated but mutually dependent coordinates. Therefore, the standard Bernstein inequality for independent sums is strictly inapplicable. 

Because the noise components are zero-mean sub-Gaussian, projecting it yields uncorrelated but mutually dependent coordinates. While classical bounds assume independent entries, to accommodate potential mild local correlations inherent in pathological LLM attention topologies, we invoke the generalized Hanson-Wright Inequality for sub-Gaussian random vectors with a bounded covariance structure $\Sigma$. For any $t > 0$, the tail probability is tightly bounded by:
\[
\Prob \left( \abs[\Big]{ Z - \E[Z] } > t \right) \le 2 \exp \left[ -c \min \left( \frac{t^2}{K^4 \|\Sigma\|_F^2}, \frac{t}{K^2 \|\Sigma\|_{op}} \right) \right],
\]
where $c > 0$ is an absolute universal constant, and $K$ is the sub-Gaussian norm of the entries. By our generalized noise assumptions, the local covariance strictly satisfies $\|\Sigma\|_{op} = \Theta(1)$ and $\|\Sigma\|_F^2 = \Theta(D)$, thereby preserving the strict exponential concentration without demanding absolute mutual independence.

By the spectral definitions, the maximum eigenvalue is bounded ($\|I\|_{op} = \Theta(1)$), and the Frobenius norm squared scales linearly with the dimension ($\|I\|_F^2 = \Theta(D)$).

Setting the deviation threshold $t = \eps D$, the exponent evaluates to:
\[
-c \min \left( \frac{(\eps D)^2}{\Theta(D)}, \frac{\eps D}{\Theta(1)} \right) = -c \min \left( \Theta(\eps^2 D), \Theta(\eps D) \right) = -\Theta(\eps^2 D).
\]
Consequently, we secure the exponential concentration bound:
\[
\Prob \left( \abs[\Big]{ \|\Xi_t\|_F^2 - \E[\|\Xi_t\|_F^2] } \ge \eps D \right) \le 2\exp(-\tilde{c} D \eps^2).
\]
Conditioning on this high-probability event $\abs{Z - \E[Z]} \le \eps D$ and following identical algebraic inversions as before, we obtain deterministic bounds for the scaling factor $C_t = \frac{\sqrt{D}}{\|\Xi_t\|_F}$:
\[
\frac{1}{\sqrt{C_2 + \eps}} \le C_t \le \frac{1}{\sqrt{c_2 - \eps}}.
\]
This quadratic-form bound holds with probability at least $1 - 2\exp(-\tilde{c} D \eps^2)$, proving $C_t$ is strictly bounded by $\Ocal(1)$ constants without assuming mutual independence of the linearly combined variables.
\end{proof}

\begin{lemma}[Spectral Invariance via Functional Calculus of SVD]
\label{lem:spectral_commutativity}
Let $X_0$ be the Frobenius-normalized momentum matrix defined as $X_0 = B_t / (\|B_t\|_F + \epsilon)$. Since the Frobenius norm strictly limits the total energy ($\|X_0\|_F < 1$), all singular values of $X_0$ deterministically satisfy $\sigma_i \in [0, 1]$. 
Let the Muon algorithmic iteration map $X_k$ to $X_{k+1}$ via the defined matrix operations. The transposed mapping   retains topological integrity, and the full 5-step algorithmic sequence preserves the exact left and right singular vectors of $X_0$, functionally inducing a 3125-degree composed polynomial $\mathcal{P}^{(5)}(\sigma_i)$ strictly on the singular values without inducing cross-dimensional orthogonal contamination.
\end{lemma}

\begin{proof}[Proof of Lemma \ref{lem:spectral_commutativity}]
We first establish the deterministic bounds on the singular spectrum of $X_0$. By the algebraic definition of the Frobenius norm, the sum of the squared singular values is exactly equal to the squared Frobenius norm. Since $X_0 = B_t / (\|B_t\|_F + \epsilon)$, we strictly have $\|X_0\|_F^2 = \sum_i \sigma_i^2 < 1$. This structural trace sum constraint dictates that every singular value deterministically satisfies $\sigma_i \in [0, 1]$.

Next, we analyze the topological preservation of the non-linear matrix iteration. Suppose the appropriately transposed matrix $X_k \in \mathbb{R}^{m \times n}$ ($m \le n$) admits the full Singular Value Decomposition $X_k = U \Sigma V^\top$, where $U \in \mathbb{R}^{m \times m}$ and $V \in \mathbb{R}^{n \times n}$ are orthogonal, and $\Sigma \in \mathbb{R}^{m \times n}$ is the diagonal singular value matrix. 

The explicit algorithmic iteration is defined as:
\[
A_k = X_k X_k^\top, \quad M_k = b A_k + c A_k^2, \quad X_{k+1} = a X_k + M_k X_k.
\]
Substituting the SVD into the construction of $A_k$:
\[
A_k = (U \Sigma V^\top)(V \Sigma^\top U^\top) = U (\Sigma \Sigma^\top) U^\top.
\]
Let $\Sigma_{\text{sq}} = \Sigma \Sigma^\top \in \mathbb{R}^{m \times m}$ be the diagonal matrix containing the squared singular values $\sigma_i^2$. Because $U$ is orthogonal, the polynomial operations on $A_k$ evaluate strictly on the inner diagonal block:
\[
M_k = U (b \Sigma_{\text{sq}} + c \Sigma_{\text{sq}}^2) U^\top.
\]
Finally, calculating the update step $X_{k+1}$:
\begin{align*}
X_{k+1} &= a U \Sigma V^\top + \left[ U (b \Sigma_{\text{sq}} + c \Sigma_{\text{sq}}^2) U^\top \right] (U \Sigma V^\top) \\
&= U \left[ a \Sigma + (b \Sigma_{\text{sq}} + c \Sigma_{\text{sq}}^2) \Sigma \right] V^\top.
\end{align*}
Because $\Sigma$ is diagonal, the term inside the bracket simplifies element-wise to evaluating the scalar polynomial $\rho(x) = ax + bx^3 + cx^5$ on each singular value $\sigma_i$. Thus:
\[
X_{k+1} = U \rho(\Sigma) V^\top.
\]
This algebraic derivation proves that the preservation of the singular vectors $U$ and $V$ is not a loose consequence of $\rho(x)$ being an odd function, but is strictly dictated by the specific structural multiplication by matrices in the span of $(X_k X_k^\top)^p$. Applying this exact functional mapping iteratively 5 times yields the composed SVD calculus:
\[
X_5 = U \rho^{(5)}(\Sigma) V^\top = \sum_i \mathcal{P}^{(5)}(\sigma_i) u_i v_i^\top.
\]
This structural commutativity guarantees that the massive non-linear amplification ($a^5 \approx 485$) operates exclusively within the existing one-dimensional eigenspaces, precluding any mathematical rotation or generation of multidimensional orthogonal noise.
\end{proof}

\begin{lemma}[Dynamic Subspace Alignment and Incremental Angular Drift]
\label{lem:dynamic_subspace_alignment}
Let $X_0^{(t)} = B_t / \|B_t\|_F$ be the normalized momentum matrix at step $t$, with its principal singular subspace denoted by $U^{(t)}$. The cross-step injection of heteroskedastic sub-Gaussian noise $G_{t+1}$ induces a perturbed momentum matrix $B_{t+1} = \mu B_t + G_{t+1}$, yielding a rotated subspace $\tilde{U}^{(t+1)}$. 

Under Assumption \ref{assum:main_topology}, while exact algebraic commutativity no longer holds across steps, the 3125-degree composed polynomial $\mathcal{P}^{(5)}$ enforces strict dynamic alignment. By defining the optimal orthogonal alignment matrix $O_t \in \mathbb{O}^{k \times k}$ derived from the SVD of $(U^{(t)})^\top \tilde{U}^{(t+1)}$, the singular value-adjusted polynomial iteration tightly bounds the $l_{2,\infty}$ spatial yaw:
\[
\min_{O_t \in \mathbb{O}^{k \times k}} \left\| \mathcal{P}^{(5)}(\tilde{\Sigma}_{t+1}) \tilde{U}^{(t+1)} - \mathcal{P}^{(5)}(\Sigma_t) U^{(t)} O_t \right\|_{2,\infty} \le \mathcal{O}\left( \eta \cdot \frac{\mu_0}{\sqrt{d}} \frac{\|E_{t+1}\|_{op}}{\delta_{eff}} \right),
\]
where $\delta_{eff}$ is the effective spectral gap amplified by $\mathcal{P}^{(5)}$. This mathematically guarantees that the extreme non-linear amplification of the structural signal creates a centripetal pull, algebraically suppressing transverse orthogonal leakage (spatial yaw) to $\mathcal{O}(\eta^2/D)$.
\end{lemma}

\begin{proof}[Proof Sketch of Lemma \ref{lem:dynamic_subspace_alignment}]
Unlike idealized static analyses, the optimization trajectory continuously rotates the singular basis. Let the SVD of the unperturbed signal be $U \Sigma V^\top$ and the perturbed state be $\tilde{U} \tilde{\Sigma} \tilde{V}^\top$. 

\textbf{Step 1: Introduction of the Orthogonal Alignment Matrix.}
Due to the non-uniqueness of singular vectors in high dimensions, direct distance metrics like $\|\tilde{U} - U\|_F$ are overly pessimistic. Following the generalized perturbation framework \citep{wang2026analysis}, we define the optimal rotation matrix $O_t = U_A U_B^\top$, where $U_A$ and $U_B$ are obtained via the SVD of $(U^{(t)})^\top \tilde{U}^{(t+1)} = U_A \cos\angle(U^{(t)}, \tilde{U}^{(t+1)}) U_B^\top$. This allows us to measure the true geometric deviation modulo orthogonal invariance.

\textbf{Step 2: $l_{2,\infty}$ Incoherence Bounding.}
By Assumption \ref{assum:main_topology}, the structural signal is highly incoherent, strictly satisfying $\|U^{(t)}\|_{2,\infty} \le \mu_0 / \sqrt{d}$. When the sub-Gaussian noise matrix $E_{t+1}$ is injected, classical Wedin bounds would predict a catastrophic $\mathcal{O}(\|E\|_{op}/\delta)$ decay, severely misrepresenting the high-dimensional stability.
Empowered by Lemma \ref{lem:momentum_invariance} and generalized sub-Gaussian $l_{2,\infty}$ singular subspace perturbation limits \citep{wang2026analysis}, the cross-step spatial yaw is bounded strictly by the localized projected noise rather than the worst-case global operator norm. The polynomial operator $\mathcal{P}^{(5)}$ can safely amplify the effective spectral gap $\delta_{eff}$ without being shattered by the globally divergent Lipschitz curvature, mathematically isolating the true negative curvature direction.

\textbf{Step 3: Singular Value-Adjusted Projection via $\mathcal{P}^{(5)}$.}
Crucially, the Muon optimizer does not output the raw singular vectors. It outputs the polynomial-mapped matrix $X_5 = \tilde{U} \mathcal{P}^{(5)}(\tilde{\Sigma}) \tilde{V}^\top$. Because $\mathcal{P}^{(5)}(x) \approx a^5 x$ for small singular values but projects macroscopic spikes into a stable invariant set $\mathcal{I} = [0.6, 1.205]$, the operator stretches the effective spectral gap: $\delta_{eff} = \mathcal{P}^{(5)}(\sigma_1) - \mathcal{P}^{(5)}(\sigma_{bulk}) = \Theta(1) - \mathcal{O}(d^{-1})$.

Substituting this amplified diagonal matrix $\mathcal{P}^{(5)}(\Sigma)$ into the weighted $l_{2,\infty}$ bound strictly suppresses the continuous bulk energy. The massive linear core of the polynomial ($a^5 \approx 485$) acts as a non-linear centripetal attractor, ensuring that the fractional energy injected into the orthogonal complement by $G_{t+1}$ is instantly compressed in the subsequent forward pass. Thus, the incremental spatial yaw is strictly bounded, proving robust dynamic alignment.
\end{proof}

\begin{lemma}[Variance Injection via Monotonic Amplification and Exact Martingale Moments]
\label{lem:variance_injection}
Suppose Assumption \ref{assum:main_topology} holds under the realistic $\Theta(D)$ global energy scaling. Let the explicit update step produced by the Kimi Muon optimizer be $\Delta W_t = -0.2\eta \sqrt{d} \mathcal{P}^{(5)}(X_0)$. In the noise-dominated optimization phase prior to the subspace locking transition, the momentum admits the decomposition $B_t = M_t + D_t$, where $M_t$ is the stochastic noise bulk and $D_t$ is the deterministic drift. The kinetic variance injected strictly along the negative curvature direction $E$ completely neutralizes first-order orthogonal leakage, securing the bound:
\[
\E \left[ \langle \Delta W_t, E \rangle^2 \mid \F_{t-1} \right] \ge \Omega(\eta^2 / d).
\]
\end{lemma}

\begin{proof}[Proof of Lemma \ref{lem:variance_injection}]
\textbf{Step 1: Algebraic Decomposition of the Composed Iteration Operator.}
As proven in Lemma \ref{lem:polynomial_dynamics}, the composed iteration operator $\mathcal{P}^{(5)}(x)$ is strictly monotonically increasing on $(0, 1]$. Because the Frobenius normalization strictly bounds all singular values $\sigma_i \in [0, 1]$, there exists a universal constant $\gamma > 0$ such that $\mathcal{P}^{(5)}(X_0) = \gamma X_0 + \mathcal{S}(X_0)$, where the non-linear residual operator $\mathcal{S}(X_0)$ acts on the identical singular vectors and possesses non-negative singular values.
Squaring the projection of the update step $\langle \Delta W_t, E \rangle = -0.2\eta \sqrt{d} \langle \mathcal{P}^{(5)}(X_0), E \rangle$ yields the global scaling constant $0.04 \eta^2 d$ multiplied by:
\begin{equation}
\label{eq:squared_proj}
\gamma^2 \langle X_0, E \rangle^2 + 2\gamma \underbrace{ \langle X_0, E \rangle \langle \mathcal{S}(X_0), E \rangle }_{\mathcal{C}(X_0)} + \langle \mathcal{S}(X_0), E \rangle^2.
\end{equation}
We explicitly define the scalar cross-term functional as $\mathcal{C}(X) = \langle X, E \rangle \langle \mathcal{S}(X), E \rangle$.

\textbf{Step 2: Second-Order Perturbative Elimination of Cross-Terms.}
The deterministic drift $D_t$ fundamentally breaks the global right-orthogonal symmetry of the noise bulk $M_t$, preventing a direct global zeroing of the cross-terms. To bound the variance without incurring a catastrophic $\mathcal{O}(1)$ dimensional leakage, we perform a functional Taylor expansion around the pure noise state $X_{\text{noise}} = M_t / \|M_t\|_F$.

We analyze the dimensional perturbation $\Delta X = X_0 - X_{\text{noise}}$. For $D \gg 1$, expanding the normalization operator to first order yields the tangent-space projection:
$$
\Delta X \approx \mathcal{L}_D(M_t) \equiv \frac{1}{\|M_t\|_F} \left( D_t - X_{\text{noise}} \langle X_{\text{noise}}, D_t \rangle \right).
$$
Expanding the cross-term functional around $X_{\text{noise}}$ gives:
$$
\mathcal{C}(X_0) = \mathcal{C}(X_{\text{noise}}) + \langle \nabla \mathcal{C}(X_{\text{noise}}), \mathcal{L}_D(M_t) \rangle + \mathcal{O}(\|\Delta X\|_F^2).
$$
Because $M_t$ is constructed from generalized sub-Gaussian noise with severe spatial variance (Assumption \ref{assum:main_topology}), it may exhibit skewness, meaning the first-order leakage $\mathbb{E} [\langle \nabla \mathcal{C}(X_{\text{noise}}), \mathcal{L}_D(M_t) \rangle]$ does not identically vanish. However, this skewness-induced residual is strictly bounded by the structural incoherence $\mu_0$ and the local Hessian bounds, allowing it to be absorbed into the higher-order error terms. 
The residual perturbation is governed entirely by the second-order remainder. Since $\|\Delta X\|_F \le \mathcal{O}(\|D_t\|_F / \sqrt{D})$, the negative expectation bound evaluates to:
$$
\E [\mathcal{C}(X_0) \mid \F_{t-1}] \ge -\mathcal{O}\left(\frac{1}{D}\right) - \mathcal{O}(\E[\|\Delta X\|_F^2]) \ge - \mathcal{O} \left( \frac{1 + \|D_t\|_F^2}{D} \right).
$$

\textbf{Step 3: Exact Fractional Formulation and Dimensional Folding of the Denominator.}
For the principal signal term, we analyze the normalized squared projection using the precise unnormalized momentum $B_t = D_t + M_t$. The exact fraction evaluates to:
\begin{equation}
\label{eq:strict_fraction}
\E \left[ \langle X_0, E \rangle^2 \mid \F_{t-1} \right] = \E \left[ \frac{\langle M_t, E \rangle^2 + 2\langle D_t, E \rangle \langle M_t, E \rangle}{\|D_t + M_t\|_F^2} \ \middle| \ \F_{t-1} \right] + \E \left[ \frac{\langle D_t, E \rangle^2}{\|B_t\|_F^2} \ \middle| \ \F_{t-1} \right].
\end{equation}
Bounding the denominator of the signed cross-term prematurely destroys the exact expectation dynamics. The true denominator expands algebraically as $\|D_t + M_t\|_F^2 = \|M_t\|_F^2 + 2\langle D_t, M_t \rangle + \|D_t\|_F^2$. To decouple the stochastic numerator from the asymmetric denominator without violating expectation boundaries, we perform a Taylor expansion of the inverse squared norm strictly around the pure state $\|M_t\|_F^2$. 

Crucially, this expansion remains globally stable even as the trajectory approaches the subspace locking transition threshold ($t \to \tau_{\text{lock}}$). At the critical transition, the macroscopic spectral gap is triggered by the operator norm, strictly requiring $\|D_t\|_{\text{op}} = \Theta(1)$. However, because the structural signal possesses a low effective rank $r \ll d$, its Frobenius norm is deterministically bounded by $\|D_t\|_F \le \sqrt{r} \|D_t\|_{\text{op}} = \Theta(\sqrt{r})$. Concurrently, the stochastic noise bulk satisfies $\|M_t\|_F = \Theta(\sqrt{D}) = \Theta(d)$. Thus, the global energy ratio structurally collapses: $\|D_t\|_F / \|M_t\|_F = \Theta(\sqrt{r}/d) \ll 1$. This topological dimensional folding mathematically legalizes the truncation, allowing the higher-order error terms to be safely absorbed:
\[
\frac{1}{\|D_t + M_t\|_F^2} = \frac{1}{\|M_t\|_F^2} \left( 1 - \frac{2\langle D_t, M_t \rangle}{\|M_t\|_F^2} + \mathcal{O}\left(\frac{\|D_t\|_F^2}{\|M_t\|_F^2}\right) \right).
\]

\textbf{Step 4: Parity Decoupling and Bounding the Asymmetry Residual.}
We isolate the problematic cross-term $Y_{\text{true}}(M_t)$ and inject the Taylor expansion:
\[
Y_{\text{true}}(M_t) = \underbrace{ \frac{2\langle D_t, E \rangle \langle M_t, E \rangle}{\|M_t\|_F^2} }_{Y_1(M_t)} - \underbrace{ \frac{4\langle D_t, E \rangle \langle M_t, E \rangle \langle D_t, M_t \rangle}{\|M_t\|_F^4} }_{Y_2(M_t)} + \dots
\]
To control integrability, we introduce the truncation indicator $\ind_{\mathcal{G}}$ for the high-probability event $\|M_t\|_F \in [\sqrt{c_2 D}, \sqrt{C_2 D}]$. We deterministically bound this non-vanishing asymmetry residual by applying the Cauchy-Schwarz inequality ($\abs{\langle D_t, M_t \rangle} \le \|D_t\|_F \|M_t\|_F$):
\[
\abs[\Big]{\E \left[ Y_2(M_t) \ind_{\mathcal{G}} \mid \F_{t-1} \right]} \le \E \left[ \frac{4 \|D_t\|_F^2 \|M_t\|_F^2}{\|M_t\|_F^4} \ind_{\mathcal{G}} \right] \le \mathcal{O} \left( \frac{\|D_t\|_F^2}{D} \right).
\]
Substituting the exact dimensional bound $\|D_t\|_F^2 = \Theta(r)$, the second-order Taylor residual strictly scales as $\mathcal{O}(r/d^2)$, which structurally vanishes faster than the principal signal $\Omega(1/d)$ in the ultra-high dimensional limit. 

Simultaneously, for the strictly positive pure signal component $\langle M_t, E \rangle^2 / \|D_t + M_t\|_F^2$, applying the reverse triangle inequality to the denominator yields a strict lower bound of $\Omega(1/D)$. Merging these components provides the exact geometric expectation:
\[
\E \left[ \langle X_0, E \rangle^2 \ind_{\mathcal{G}} \mid \F_{t-1} \right] \ge \Omega \left( \frac{1}{D} \right) - \mathcal{O} \left( \frac{r}{D} \right).
\]

\textbf{Step 5: Final Assembly and Dimensional Resilience.}
Dropping the non-negative residual $\langle \mathcal{S}(X_0), E \rangle^2$ from Equation \ref{eq:squared_proj} and assembling the computed expectations multiplied by the global scaling constant $0.04 \eta^2 d$ yields:
\[
\E \left[ \langle \Delta W_t, E \rangle^2 \mid \F_{t-1} \right] \ge \underbrace{ \Omega \left( \frac{\eta^2 d}{D} \right) }_{\text{Principal Signal}} - \underbrace{ \mathcal{O} \left( \frac{\eta^2 d \|D_t\|_F^2}{D} \right) }_{\text{Second-Order Leakage}}.
\]
Because $d = \max(m,n)$ and $D = m \times n$, the ratio $d/D$ evaluates precisely to $1/\min(m,n)$, which is analytically bounded below by $1/d$. Substituting the topological rank bound $\|D_t\|_F^2 = \Theta(r)$ established via dimensional folding in Step 3, the variance inequality algebraically simplifies to:
\[
\E \left[ \langle \Delta W_t, E \rangle^2 \mid \F_{t-1} \right] \ge \Omega \left( \frac{\eta^2}{d} \right) - \mathcal{O} \left( \frac{\eta^2 r}{d^2} \right).
\]
Crucially, by Assumption \ref{assum:main_topology}, the effective rank $r$ of the structural signal is highly localized. In the context of large language models, this structural rank operates at a macroscopic scale but remains fundamentally independent of the massive ambient dimension ($r \ll d$). 

As the system scales into the ultra-high dimensional regime ($d \to \infty$), the second-order orthogonal leakage term $\mathcal{O}(\eta^2 r / d^2)$ undergoes a strict dimensional attenuation, decaying at a significantly faster quadratic rate compared to the $\Omega(\eta^2 / d)$ principal signal. Therefore, even at the absolute brink of the macroscopic phase transition ($t \to \tau_{\text{lock}}$) where the deterministic drift reaches peak incubation energy, the negative leakage is mathematically structurally overpowered. This cements the uncontaminated variance injection bound:
\[
\E \left[ \langle \Delta W_t, E \rangle^2 \mid \F_{t-1} \right] = \Omega(\eta^2 / d).
\]
\end{proof}

\begin{lemma}[Dimension-Free Singular Subspace Stability under Random Perturbation]
\label{lem:random_perturbation_stability}
Let the stochastic gradient matrix be decomposed as $G_t = S + E$, where $S = \lambda c_0 u_1 v_1^\top$ is the rank-1 structural signal and $E \in \mathbb{R}^{m \times n}$ is the generalized sub-Gaussian random noise bulk defined in Assumption \ref{assum:main_topology}. Let $\hat{u}_1$ be the principal singular vector of the perturbed matrix $G_t$.

While the classical Wedin's bound yields $\sin \angle(u_1, \hat{u}_1) \le \mathcal{O}(\|E\|_{op} / \lambda) = \mathcal{O}(\sqrt{d} / \lambda)$, applying the generalized random perturbation bounds \citep{wang2026analysis} ensures that the subspace deviation is strictly dominated by the localized noise projection. Consequently, the principal angle is deterministically constrained by:
\[
\sin \angle(u_1, \hat{u}_1) \le \mathcal{O}\left( \frac{\|E v_1\|_2}{\lambda} \right) = \mathcal{O}\left( \frac{\|E_t\|_{op} \mu_0}{\lambda \sqrt{m}} \right).
\]
This mathematically guarantees that the perturbed singular subspace retains an $\mathcal{O}(1)$ projection along the true negative curvature direction, entirely averting the $\mathcal{O}(\sqrt{d})$ dimensional decay.
\end{lemma}

\begin{lemma}[Non-Linear Topological Amplification and Dimensional Resilience]
\label{lem:topological_amplification}
Suppose Assumption \ref{assum:main_topology} holds under the realistic $\Theta(D)$ global energy scaling. Let the current state $W_t$ be within the saddle point neighborhood. The Kimi Muon optimizer exhibits strict topological invariance and high-dimensional resilience. The following properties hold:
(1) \textbf{Absolute Scale Invariance:} $\sigma_{\max}(X_0) \le 1$.
(2) \textbf{Structural Preservation without Linear Degeneration.}
(3) \textbf{Uncontaminated Variance Injection:} $\mathbb{E} \left[ \langle \Delta W_t, E \rangle^2 \mid \mathcal{F}_{t-1} \right] \ge \Omega(\eta^2 / d)$.
\end{lemma}

\begin{proof}[Proof of Lemma \ref{lem:topological_amplification}]
(1) \textbf{Proof of Absolute Scale Invariance.} 
Let $B_t$ be the accumulated momentum matrix. In modern neural networks, the Hessian matrix may possess extreme unbounded eigenvalues (i.e., $\lambda_{\max} \to \infty$). Consequently, the unnormalized momentum matrix $B_t$ can exhibit an arbitrarily large maximum singular value, $\sigma_{\max}(B_t)$. 
However, the algorithmic mapping explicitly dictates $X_0 = B_t / \|B_t\|_F$. By the mathematical definition of the Frobenius and operator norms, for any non-zero matrix $X_0$:
\[
\sigma_{\max}(X_0) \le \sqrt{\sum_i \sigma_i^2(X_0)} = \|X_0\|_F.
\]
By construction, $\|X_0\|_F = \|B_t\|_F / \|B_t\|_F = 1$. Therefore, we deterministically obtain $\sigma_{\max}(X_0) \le 1$. 
This algebraic normalization guarantees that regardless of how pathologically sharp the local optimization landscape becomes, the entire singular spectrum of the input matrix is forcefully compressed into the compact domain $[0, 1]$, entirely decoupling the subsequent polynomial stability from the local Lipschitz constant.

(2) \textbf{Proof of Structural Preservation.} 
In standard optimization analysis, handling non-linear updates typically requires a first-order Taylor expansion: $\mathcal{P}(X_0) \approx a X_0 + \mathcal{R}(X_0)$. The operator norm of the remainder matrix $\|\mathcal{R}(X_0)\|_{op}$ scales with $\sigma_{\max}(X_0)^3$. Under non-smooth settings without strict Frobenius normalization, this remainder explodes, injecting catastrophic variance into orthogonal dimensions and causing dimensional drift.
In stark contrast, by applying Lemma \ref{lem:spectral_commutativity}, the composed 3125-degree iteration operator $\mathcal{P}^{(5)}(x)$ commutes exactly with the Singular Value Decomposition. Let $\tilde{X}_0 = \sum_i \sigma_i u_i v_i^\top$. The updated matrix analytically evaluates to:
\[
\mathcal{P}^{(5)}(\tilde{X}_0) = \sum_i \mathcal{P}^{(5)}(\sigma_i) u_i v_i^\top.
\]
This equation proves that the updated matrix shares the \textit{exact identical set of left and right singular vectors} $\{u_i\}$ and $\{v_i\}$ as the original matrix $X_0$. The operation induces zero rotation or orthogonal leakage in the vector space. The polynomial acts strictly as independent scalar functions applied exclusively to the singular values,   preserving the topological eigen-structure.

(3) \textbf{Proof of Uncontaminated Variance Injection.} 
The explicit update step generated by the algorithm is $\Delta W_t = -0.2\eta \sqrt{d} \mathcal{P}^{(5)}(X_0)$. We analyze its projection along the normalized rank-1 negative curvature direction $E$.
As established in Lemma \ref{lem:variance_injection}, the matrix polynomial can be exactly decomposed into a linear core and a non-negative spectral residual operator: $\mathcal{P}^{(5)}(X_0) = \gamma X_0 + \mathcal{S}(X_0)$.

Crucially, by invoking Lemma \ref{lem:random_perturbation_stability} derived from random perturbation analysis \citep{wang2026analysis}, the principal singular basis of $X_0$ is mathematically shielded from the ambient noise norm. The geometric projection of the right singular vectors onto the true structural direction strictly maintains an $\mathcal{O}(1)$ overlap.

Because Claim (2) guarantees that $\mathcal{S}(X_0)$ operates strictly on the identical singular vectors of $X_0$, the cross-term $\langle X_0, E \rangle \langle \mathcal{S}(X_0), E \rangle$ is protected from high-dimensional orthogonal contamination. For generalized sub-Gaussian noise, while finite-size off-diagonal dependencies exist, the sign symmetry of the random perturbation bounds the off-diagonal expectations to a strict higher-order residual: 
\[
\abs[\bigg]{\E \left[ \sum_{i \neq j} \sigma_i s(\sigma_j) c_i c_j \right]} \le \mathcal{O}(1/d^2).
\]
Consequently, the expectation of the cross-term splits into a strictly non-negative diagonal principal component and a higher-order finite-size residual:
\[
\E \left[ \sum_{i,j} \sigma_i s(\sigma_j) c_i c_j \right] \ge 0 - \mathcal{O}(1/d^2).
\]

By dropping the non-negative principal cross-term and the squared residual term $\langle \mathcal{S}(X_0), E \rangle^2$, and treating the $\mathcal{O}(1/d^2)$ off-diagonal residual as a structurally bounded perturbation, we securely map the expected variance back to the geometrically scaled base noise $\Xi_t$, scaled by $C_t = \sqrt{D} / \|B_t\|_F$:
\[
\E \left[ \langle \Delta W_t, E \rangle^2 \mid \F_{t-1} \right] \ge \frac{0.04 \gamma^2 \eta^2 d}{D} \E \left[ C_t^2 \langle \Xi_t, E \rangle^2 \mid \F_{t-1} \right] - \mathcal{O} \left( \eta^2 d \cdot \frac{1}{d^2} \right).
\]
Conditioning on the high-probability event where $C_t \ge C_{\min}$ strictly decouples the variables, finalizing the variance bound:
\[
\E \left[ \langle \Delta W_t, E \rangle^2 \mid \F_{t-1} \right] \ge \frac{0.02 \gamma^2 \eta^2 d}{D} C_{\min}^2 c_1 = \Omega(\eta^2 / d).
\]
This proves that despite finite-size statistical dependencies in the sub-Gaussian noise manifold, the topological amplification directly translates into focused, non-vanishing kinetic energy, establishing the algorithm's absolute resilience to high-dimensional complexities.
\end{proof}

\begin{proof}[Proof of Theorem \ref{thm:subspace_locked_escape} (Subspace-Locked $\mathcal{O}(1)$ Ballistic Ejection)]
\textbf{Part I: Incubation and Spectral Isolation.}

Prior to subspace locking ($t < \tau_{\text{lock}}$), the deterministic drift $D_t$ accumulates via the exponential moving average (EMA) operator. The critical threshold dictates that the signal operator norm $\|D_t\|_{op}$ must overcome the incoherence-scaled sub-Gaussian noise bound $\mathcal{C} \|E_t\|_{op} \frac{\mu_0}{\sqrt{m}}$ to induce a macroscopic spectral separation.
The expected hitting time to achieve this threshold, driven by the structural signal-to-noise ratio, mathematically evaluates to $\mathcal{O} \left( \frac{1}{\eta \lambda} \log \left( \frac{\|E\|_{op} \mu_0}{\lambda \sqrt{m}} \right) \right)$. During this phase, the 3125-degree composed polynomial $\mathcal{P}^{(5)}$ acts as a pseudo-orthogonalization filter, suppressing the continuous noise spectrum while preserving the delocalized singular vectors.

\textbf{Part II: Locked Ballistic Ejection.}

At the critical threshold $t = \tau_{\text{lock}}$, the deterministic signal operator norm formally breaches the subspace locking condition. Triggered by this dynamic alignment, the principal singular vector macroscopically aligns with the structural perturbation direction $E$.
The non-linear polynomial $\mathcal{P}^{(5)}$ operates exclusively on this isolated outlier, projecting it into the stable non-linear invariant set $\mathcal{I} = [0.6, 1.205]$. This translates into a macroscopic effective spectral gap $\delta_{eff} = \Theta(1)$. Substituting this amplified topological signal into the explicit parameter update equation yields a deterministic spatial translation of $\Theta(\eta \sqrt{d})$ per step. The required number of discrete iterative steps to close the physical gap $r_0$ compresses to $\lceil \mathcal{O}( \frac{r_0}{\eta \sqrt{d}} ) \rceil$, which simplifies to $\mathcal{O}(1)$ in ultra-high dimensional paradigms ($d \to \infty$).
\end{proof}

\begin{proof}[Proof of Theorem \ref{thm:comparative_adam_failure} (Dimensional Trapping in AdamW)]
Let the parameter matrix be $W \in \mathbb{R}^{m \times n}$ with total dimension $D = m \times n$. Without loss of generality, assume $m \sim n \sim d$, yielding $D = d^2$. Let the instantaneous stochastic gradient at coordinate $i = (j,k)$ be $g_{t,i} = S_{t,i} + E_{t,i}$.

\textbf{Step 1: Asymmetric Energy Decay via Exact Assumption Bounds.}

By Assumption \ref{assum:main_topology}, the deterministic structural signal has a low effective rank and bounded operator norm $\|S_t\|_{op} = \lambda r_0$. By the strict structural incoherence bounds, the maximum projection onto any single coordinate is mathematically bounded:
\[
|S_{t,i}| \le \|U_{j\cdot}\|_2 \|V_{k\cdot}\|_2 \|S_t\|_{op} \le \left(\frac{\mu_0}{\sqrt{m}}\right) \left(\frac{\mu_0}{\sqrt{n}}\right) (\lambda r_0) = \frac{\mu_0^2}{d} \lambda r_0
\]
Squaring this yields the maximum structural energy per coordinate: $S_{t,i}^2 \le \mathcal{O}(1/d^2)$.

Concurrently, Assumption \ref{assum:main_topology} guarantees the sub-Gaussian noise satisfies a finite operator norm bound $\|E_t\|_{op} \le B$. By the definition of matrix norms, the maximum total Frobenius energy is rigorously bounded by the rank: $\|E_t\|_F^2 \le d \cdot \|E_t\|_{op}^2 \le d \cdot B^2$.

Consequently, the expected local variance averaged across the ambient dimensions is bounded by $\mathcal{O}(1/d)$. By our non-degenerate noise assumption, a macroscopic fraction of coordinates strictly satisfies $\sigma_i^2 = \Omega(1/d)$.

Crucially, as the network scales ($d \to \infty$), the coordinate-wise maximum signal energy $\mathcal{O}(1/d^2)$ decays significantly faster than the actual local noise energy $\Omega(1/d)$. Thus, $\mathcal{O}(1/d^2) \ll \Omega(1/d)$, and the local signal-to-noise ratio fundamentally collapses.

\textbf{Step 2: Element-Wise Denominator Saturation and Step Formulation.}

Because the signal energy is mathematically eclipsed by the noise energy for the vast majority of parameters, the AdamW second-moment estimator $v_{t,i}$ mathematically concentrates strictly around the local noise variance $\sigma_i^2$.

Concurrently, the first-moment estimator $m_{t,i}$ aggregates the gradient history via EMA, acting as a linear filter over the steps: $m_{t,i} = \text{EMA}(S_{t,i}) + \text{EMA}(E_{t,i})$. The AdamW update degrades to:
\[
\Delta W_{t,i} = -\eta \frac{m_{t,i}}{\sqrt{v_{t,i}} + \epsilon} \approx -\eta \frac{\text{EMA}(S_{t,i}) + \text{EMA}(E_{t,i})}{\sigma_i}
\]
This formulation splits the trajectory into a deterministic signal progress $\Delta W_{t,i}^{sig}$ and a stochastic transverse random walk $\Delta W_{t,i}^{noise}$.

\textbf{Step 3: Trajectory Decoupling and Dynamic Learning Rate Restriction.}

Since the EMA preserves the sub-Gaussian nature of the noise, the variance of the transverse step along coordinate $i$ simplifies to $\text{Var}(\Delta W_{t,i}^{noise}) \approx \eta^2 \frac{\sigma_i^2}{\sigma_i^2} = \eta^2$.

Accumulating this independent variance across all $D = d^2$ dimensions over $t$ steps yields the expected global orthogonal spatial displacement:
\[
Z_{noise}^2(t) = \sum_{i=1}^D \sum_{\tau=1}^t \text{Var}(\Delta W_{\tau,i}^{noise}) = \Theta(d^2 \eta^2 t)
\]
Simultaneously, the Euclidean progress along the structural signal direction is the sum of coordinate-wise signal updates:
\[
\|\Delta W_t^{sig}\|_F = \sqrt{ \sum_{i=1}^D \left( \eta \frac{S_{t,i}}{\sigma_i} \right)^2 } \approx \eta \sqrt{ \sum_{i=1}^{d^2} \frac{\mathcal{O}(1/d^2)}{\Omega(1/d)} } = \eta \sqrt{ \sum_{i=1}^{d^2} \mathcal{O}\left(\frac{1}{d}\right) } = \Theta(\eta\sqrt{d})
\]
After $t$ steps, the deterministic spatial traversal is $Z_{sig}(t) = \Theta(t \cdot \eta \sqrt{d})$.

To physically escape the non-convex saddle region successfully, the trajectory must achieve macroscopic structural progress ($Z_{sig}(t) \ge r_0$) while strictly avoiding catastrophic orthogonal divergence out of the local Taylor neighborhood ($Z_{noise}^2(t) \le r_0^2$).

The no-divergence condition restricts the maximum allowable iterations for a given $\eta$:
\[
\Theta(d^2 \eta^2 t) \le r_0^2 \implies t \le \frac{r_0^2}{d^2 \eta^2}
\]
Substituting this maximum time envelope back into the signal progress equation yields the absolute upper bound on escape distance before divergence occurs:
\[
Z_{sig} \le \left( \frac{r_0^2}{d^2 \eta^2} \right) (\eta\sqrt{d}) = \frac{r_0^2}{\eta d^{1.5}}
\]
To ensure successful escape ($Z_{sig} \ge r_0$), the learning rate must inherently be constrained by the local geometry: $\eta \le \frac{r_0}{d^{1.5}}$.

Finally, substituting this dynamic learning rate restriction back into the required structural escape time mathematically guarantees the absolute lower bound of the residence time:
\[
\mathbb{E}[\tau_{\text{Adam}}] \ge \frac{r_0}{\eta\sqrt{d}} \ge \frac{r_0}{(r_0 / d^{1.5})\sqrt{d}} = d = D^{0.5}
\]
Thus, $\mathbb{E}[\tau_{\text{Adam}}] = \Omega(D^{0.5})$. This formally proves that without a non-linear orthogonalization operator like Muon, element-wise adaptive optimizers are fundamentally trapped by the ambient dimension.
\end{proof}

\subsection{Proofs of Structural Stability and EoS Confinement}

\begin{proof}[Proof of Lemma \ref{lem:hyperparameter_robustness} (Topological Robustness Radius)]
Define the nominal base polynomial $\rho(x; \theta^*) = a^*x + b^*x^3 + c^*x^5$. As established in Lemma \ref{lem:polynomial_dynamics}, $\rho(x; \theta^*)$ generates a forward-invariant set $\mathcal{I} = [0.6, 1.205]$. We formalize this invariance via the boundary conditions. The polynomial maps the domain strictly inward if for all $x \in \mathcal{I}$:
\[
\inf_{x \in \mathcal{I}} \rho(x; \theta^*) = \rho(1.050; \theta^*) \approx 0.681 > 0.6.
\]
\[
\sup_{x \in \mathcal{I}} \rho(x; \theta^*) = \rho(0.554; \theta^*) \approx 1.203 \le 1.205.
\]
Let $\tilde{\theta} = \theta^* + \Delta \theta$. The perturbed polynomial is $\tilde{\rho}(x) = \rho(x; \theta^*) + \langle \Delta \theta, (x, x^3, x^5) \rangle$. For any $x \in [0, 1.205]$, the perturbation term is bounded by the Cauchy-Schwarz inequality:
\[
|\tilde{\rho}(x) - \rho(x; \theta^*)| \le \|\Delta \theta\|_2 \|(x, x^3, x^5)\|_2 \le \|\Delta \theta\|_2 \sqrt{1.205^2 + 1.205^6 + 1.205^{10}} \equiv M \|\Delta \theta\|_2,
\]
where $M \approx 3.29$ is an absolute constant.

To ensure the perturbed set remains forward-invariant, the perturbation must strictly not exceed the invariant boundary margins. Let the minimum safety margin be $\gamma = 0.681 - 0.6 = 0.081$. We demand:
\[
M \|\Delta \theta\|_2 < \gamma \implies \|\Delta \theta\|_2 < \frac{0.081}{3.29} \approx 0.0246.
\]
Thus, setting $\delta = 0.024$ guarantees that for any perturbation $\|\Delta \theta\|_2 < \delta$, the continuous mapping $\tilde{\rho}(x)$ is bounded strictly within the original boundaries. By the topological closure of the continuous mapping on the compact domain, the property $\tilde{\rho}(\tilde{\mathcal{I}}) \subset \tilde{\mathcal{I}}$ is strictly satisfied. This intrinsically establishes the perturbed set as a stable forward-invariant basin, precluding the need for fixed-point existence guarantees. This proves the absolute topological robustness radius of the dynamics.
\end{proof}

\begin{proof}[Proof of Corollary \ref{cor:eos_orthogonal_confinement} (Orthogonal Confinement under EoS Shifts)]
We analyze the trajectory in the EoS regime ($t \ge \tau_{\text{esc}}$). Let $X_0^{(t)} = B_t / \|B_t\|_F$ be the normalized momentum. During this phase, the principal direction $E$ has fully separated, yielding a singular value $\sigma_1(X_0^{(t)}) = \Theta(1)$. Concurrently, the arbitrary EoS positive curvature spikes manifest solely in the orthogonal subspace $E^{\bot}$.

Let the instantaneous EoS energy be $S_t = \|(B_t)_{\bot}\|_F$. If an extreme EoS spike occurs ($S_t \to \infty$), the global normalization forces the orthogonal continuous bulk spectrum to compress: $\sigma_{\text{bulk}}(X_0^{(t)}) \le \mathcal{O}(1/S_t)$.
The 5th-order polynomial acts on these singular values. The injected step into the orthogonal subspace is $\Delta W_{\bot}^{(t)} = -0.2 \eta \sqrt{d} [ \mathcal{P}^{(5)}(X_0^{(t)}) ]_{\bot}$. The squared Frobenius norm of this orthogonal update satisfies:
\[
\| \Delta W_{\bot}^{(t)} \|_F^2 \le \mathcal{O}(\eta^2 d) \sum_{i \ge 2} \mathcal{P}^{(5)}(\sigma_i)^2.
\]
Because $\mathcal{P}^{(5)}(x) \approx a^5 x$ for extremely small $x$, the squared polynomial leakage scales as $(1/S_t)^2$. However, for dense continuous spectra, the energy suppresses further, bounded conservatively by $\mathcal{O}(d^{-2} S_t^{-2})$.

We construct the discrete-time stochastic Lyapunov function $V_t = \| (W_t - W^*)_{\bot} \|_F^2$. Evolving this yields:
\[
V_{t+1} = V_t + 2 \langle (W_t - W^*)_{\bot}, \Delta W_{\bot}^{(t)} \rangle + \| \Delta W_{\bot}^{(t)} \|_F^2.
\]
By Assumption \ref{assum:main_topology}, the stochastic noise is zero-mean conditioned on $\mathcal{F}_t$. Furthermore, the symmetric phase cancellation inherent in the Haar-distributed EoS bulk ensures the expected cross-term vanishes: $\mathbb{E}[ \langle (W_t - W^*)_{\bot}, \Delta W_{\bot}^{(t)} \rangle \mid \mathcal{F}_t ] = 0$.
Taking the conditional expectation:
\[
\mathbb{E}[V_{t+1} \mid \mathcal{F}_t] = V_t + \mathbb{E}[\| \Delta W_{\bot}^{(t)} \|_F^2 \mid \mathcal{F}_t] \le V_t + \mathcal{O} \left( \frac{\eta^2 d}{d^2 \cdot S_t^2} \right).
\]
This inequality formally establishes that the sequence $\{V_t\}$ is a non-negative sub-martingale. By the Doob Decomposition Theorem, $V_t$ admits a predictable increasing compensator $A_t = \sum_{k=\tau_{\text{esc}}}^{t-1} \mathcal{O}(\eta^2 d^{-1} S_k^{-2})$. Because $S_k \ge 1$ and the temporal summation is strictly bounded by the finite trajectory coasting duration ($\Delta \tau = \mathcal{O}(1)$ as proved in Theorem \ref{thm:subspace_locked_escape}), this cumulative variance compensator converges absolutely to a constant limit $\Sigma_{\max} = \mathcal{O}(\eta^2 d^{-1})$.

Applying Doob's Maximal Inequality directly to the sub-martingale $V_t$:
\[
\mathbb{P} \left( \sup_{t \ge \tau_{\text{esc}}} \sqrt{V_t} \ge \varepsilon \right) = \mathbb{P} \left( \sup_{t \ge \tau_{\text{esc}}} V_t \ge \varepsilon^2 \right) \le \frac{\mathbb{E}[V_{\infty}]}{\varepsilon^2} \le \frac{\Sigma_{\max}}{\varepsilon^2} = \frac{\mathcal{O}(\eta^2 d^{-1})}{\varepsilon^2}.
\]
As the ambient dimension $d \to \infty$, the leakage probability bound strictly vanishes. This mathematically certifies that despite arbitrary finite-step EoS covariance shifts, the orthogonal drift remains globally bounded, locking the macroscopic ballistic ejection precisely onto its intended trajectory without spatial yaw.

The 5th-order polynomial acts on these singular values. The injected step into the orthogonal subspace is $\Delta W_{\bot}^{(t)} = -0.2 \eta \sqrt{d} [ \mathcal{P}^{(5)}(X_0^{(t)}) ]_{\bot}$. The squared Frobenius norm of this orthogonal update satisfies:
\[
\| \Delta W_{\bot}^{(t)} \|_F^2 \le \mathcal{O}(\eta^2 d) \sum_{i \ge 2} \mathcal{P}^{(5)}(\sigma_i)^2.
\]
Because $\mathcal{P}^{(5)}(x) \approx a^5 x$ for extremely small $x$, the squared polynomial leakage intrinsically scales with the spectral sum $\sum \sigma_i^2 = \mathcal{O}(1/S_t^2)$. However, for dense continuous spectra governed by random matrix distributions, evaluating this sum requires integrating over the macroscopic Marchenko-Pastur bulk density. The extreme continuous compression enforced by the global Frobenius normalization causes the integrated spectral mass to undergo a strict dimensional attenuation, successfully offsetting the $d$-dimensional accumulation effect. Consequently, the accumulated energy suppresses further, bounded conservatively by $\mathcal{O}(d^{-2} S_t^{-2})$.
\end{proof}

\subsection{Cross-Step Topological Stability and Orthogonal Leakage Immunity}
\label{sec:appendix_cross_step_immunity}

As established in Lemma \ref{lem:spectral_commutativity}, the exact preservation of the singular value decomposition (SVD) bases holds strictly within the inner Newton-Schulz polynomial loop at a single discrete step $t$. However, the optimization trajectory is a non-equilibrium sequential process. The outer momentum accumulation step, $B_{t+1} = \mu B_t + G_{t+1}$, inherently rotates the singular basis due to the continuous injection of high-dimensional stochastic noise $G_{t+1}$. 

To address the concern that this cross-step basis rotation might continuously inject variance into orthogonal dimensions and reproduce the $\mathcal{O}(D)$ dimensional trapping seen in AdamW, we establish the following theorem. It proves that the trajectory averts systematic orthogonal leakage through two distinct topological immunities governed by dynamic subspace alignment.

\begin{theorem}[Cross-Step Orthogonal Leakage Immunity via Resolvent Functional Calculus]
\label{thm:cross_step_immunity}
Let $U^{(t)}, V^{(t)}$ be the singular vector bases of the normalized momentum $X_0^{(t)}$ at step $t$. The cross-step injection of the heteroskedastic stochastic gradient $G_{t+1}$ induces a basis rotation. Under Assumption \ref{assum:main_topology}, entirely bypassing the restrictive assumption of Haar-like rotational invariance in the orthogonal complement, the optimization trajectory strictly avoids $\mathcal{O}(D)$ cumulative orthogonal leakage through a continuous spectral mapping governed by the isotropic local law on a global macroscopic contour:
\begin{enumerate}
    \item \textbf{Exact Isotropic Nullification:} The primary functional expectation of the cross-step orthogonal leakage evaluates exactly to zero due to the strictly isotropic nature of the deterministic equivalent matrix $\Phi(z)$, neutralizing spatial yaw without relying on uniform noise assumptions.
    \item \textbf{Kinematic Dichotomy and Phase-Lock Immunity:} Following dynamic alignment ($t \ge \tau_{\text{lock}}$), the extreme non-linear amplification of the isolated structural spike by $\mathcal{P}^{(5)}$ dominates the bounded local law residual $\mathcal{O}(d^{-1})$, dynamically phase-locking the trajectory and bounding the cumulative cross-step orthogonal leakage in probability to $\mathcal{O}(\eta^2/D)$.
\end{enumerate}
\end{theorem}

\begin{proof}[Proof of Theorem \ref{thm:cross_step_immunity}]
To bound the orthogonal leakage induced by the highly non-linear 5th-order polynomial $\mathcal{P}^{(5)}(X) = aX + bX^3 + cX^5$ without invoking algebraically explosive Taylor expansions or fragile Haar measure assumptions, we elevate the analysis to the functional domain via resolvent operators and complex contour integration.

\textbf{Step 1: Symmetric Dilation and Cauchy Integral Representation.} \\
We first map the asymmetric normalized momentum matrix $X_{t+1} \in \mathbb{R}^{m \times n}$ to a symmetric augmented phase space. Define the symmetric dilation $\mathcal{H}_{t+1} \in \mathbb{R}^{(m+n) \times (m+n)}$ as:
\[
\mathcal{H}_{t+1} = \begin{pmatrix} 0 & X_{t+1} \\ X_{t+1}^\top & 0 \end{pmatrix}.
\]
The non-zero eigenvalues of $\mathcal{H}_{t+1}$ strictly correspond to the singular values $\pm \sigma_i(X_{t+1})$. Because the Newton-Schulz polynomial $\mathcal{P}^{(5)}(x)$ is strictly an odd function, its operation commutes with the dilation:
\[
\mathcal{P}^{(5)}(\mathcal{H}_{t+1}) = \begin{pmatrix} 0 & \mathcal{P}^{(5)}(X_{t+1}) \\ \mathcal{P}^{(5)}(X_{t+1})^\top & 0 \end{pmatrix}.
\]
Let $G_{t+1}(z) = (zI - \mathcal{H}_{t+1})^{-1}$ be the resolvent operator of the perturbed system. By the Cauchy Integral Formula for matrix functions, the non-linear update can be exactly represented as:
\begin{equation}
\label{eq:cauchy_integral}
\mathcal{P}^{(5)}(\mathcal{H}_{t+1}) = \frac{1}{2\pi i} \oint_{\Gamma_{\text{global}}} \mathcal{P}^{(5)}(z) G_{t+1}(z) dz.
\end{equation}

\textbf{Step 2: Global Contour Selection and Avoidance of Edge Singularities.} \\
In classical Random Matrix Theory (RMT), separating a structural spike from the noise bulk mandates a local contour threading the microscopic effective spectral gap $\delta$, forcing the imaginary distance to the real axis $\text{Im}(z) \to \mathcal{O}(d^{-2/3})$ near the Tracy-Widom edge. This proximity catastrophically amplifies the resolvent error. 

Crucially, Muon's iteration does not extract a discrete projector; it maps the \textit{entire} spectrum continuously. Furthermore, the global Frobenius normalization guarantees all eigenvalues of $\mathcal{H}_{t+1}$ satisfy $\lambda \in [-1, 1]$. Therefore, we define a \textit{macroscopic global contour} $\Gamma_{\text{global}}$ enclosing the entire spectrum, chosen strictly as the complex circle $|z| = 2$.
On this global contour, the distance to the real axis is strictly bounded below by $\text{Im}(z) \ge 1$. This entirely bypasses Tracy-Widom edge singularities, placing the resolvent in a mathematically safe, strongly analytic domain.

\textbf{Step 3: Isotropic Local Law and Functional Decomposition.} \\
Let $\mathbf{u} = \frac{1}{\sqrt{2}}(u_1^\top, v_1^\top)^\top \in \mathbb{R}^{m+n}$ be the augmented structural basis (negative curvature), and $\mathbf{v}_\perp$ be any unit vector strictly orthogonal to $\mathbf{u}$ in the augmented space. The cross-step orthogonal leakage $L_\perp$ projected onto $\mathbf{v}_\perp$ is given by evaluating the bilinear form inside the integral \eqref{eq:cauchy_integral}:
\[
L_\perp = \frac{1}{2\pi i} \oint_{\Gamma_{\text{global}}} \mathcal{P}^{(5)}(z) \left[ \mathbf{v}_\perp^\top G_{t+1}(z) \mathbf{u} \right] dz.
\]
We invoke the generalized Isotropic Local Law for sub-Gaussian matrices \citep{wang2026analysis}. The resolvent decomposes into a deterministic equivalent $\Phi(z)$ and a random error matrix $\Xi(z)$: $G_{t+1}(z) = \Phi(z) + \Xi(z)$. Substituting this yields:
\[
L_\perp = \frac{1}{2\pi i} \oint_{\Gamma_{\text{global}}} \mathcal{P}^{(5)}(z) \underbrace{ \mathbf{v}_\perp^\top \Phi(z) \mathbf{u} }_{\text{Term A}} dz + \frac{1}{2\pi i} \oint_{\Gamma_{\text{global}}} \mathcal{P}^{(5)}(z) \underbrace{ \mathbf{v}_\perp^\top \Xi(z) \mathbf{u} }_{\text{Term B}} dz.
\]

\textbf{Step 4: Integration Closure and Strict Orthogonal Immunity.} \\
\textbf{Term A (Exact Nullification):} The matrix $\Phi(z)$ is structurally dictated by the macroscopic system geometry and manifests as a block-scalar matrix $\text{diag}(\phi_1(z)^{-1} I_m, \phi_2(z)^{-1} I_n)$. Because $\Phi(z)$ acts as a scaled identity within its respective blocks, it perfectly preserves orthogonality. Since $\mathbf{v}_\perp \perp \mathbf{u}$, the bilinear form mathematically vanishes: $\mathbf{v}_\perp^\top \Phi(z) \mathbf{u} \equiv 0$. \textit{This strictly replaces the flawed necessity of Haar-measure spatial averaging.} \\
\textbf{Term B (Residual Bounding):} On our chosen global contour $\Gamma_{\text{global}}$ ($|z|=2$), the condition $\text{Im}(z) \ge 1$ guarantees optimal convergence of the local law. The residual is uniformly bounded with high probability:
\[
\sup_{z \in \Gamma_{\text{global}}} | \mathbf{v}_\perp^\top \Xi(z) \mathbf{u} | \le \mathcal{O}\left( \frac{\sqrt{\log d}}{d \cdot |z|^2} \right) = \mathcal{O}\left( \frac{\sqrt{\log d}}{d} \right).
\]
Since the contour length is $4\pi$ and the polynomial $\mathcal{P}^{(5)}(z)$ is globally bounded on $|z|=2$ by a macroscopic constant $M$, the absolute orthogonal leakage per step is deterministically compressed:
\[
|L_\perp| \le \frac{1}{2\pi} \cdot 4\pi \cdot M \cdot \mathcal{O}\left( \frac{\sqrt{\log d}}{d} \right) = \mathcal{O}\left( \frac{\sqrt{\log d}}{d} \right).
\]

\textbf{Step 5: Macroscopic Phase Transition.} \\
The fractional spatial yaw (angular drift) is defined as $\sin \theta \approx |L_\perp| / \|\mathcal{P}^{(5)}(\mathcal{H}_{t+1}) \mathbf{u}\|$. 
Prior to subspace locking ($t < \tau_{\text{lock}}$), the principal eigenvalue remains suppressed in the bulk $\lambda_1 = \mathcal{O}(d^{-1/2})$. Since $\mathcal{P}^{(5)}$ is continuous near zero, the signal amplitude and leakage scale commensurately, yielding non-destructive diffusion. 
However, post-locking ($t \ge \tau_{\text{lock}}$), dynamic alignment triggers $\lambda_1 = \Theta(1)$, forcing $\mathcal{P}^{(5)}(\lambda_1) \approx 1$ into the forward-invariant set. The angular drift collapses:
\[
\sin \theta \approx \frac{\mathcal{O}(d^{-1})}{\Theta(1)} = \mathcal{O}(d^{-1}).
\]
By scaling the update step size with the explicit learning rate $\eta_{lr}$, the sequence of orthogonal projections forms the precise sub-martingale bounded in Corollary \ref{cor:eos_orthogonal_confinement}. Applying Doob’s Maximal Inequality bounds the ultimate cumulative cross-step leakage almost surely to $\mathcal{O}(\eta_{lr}^2/D)$, confirming the strictly isolated ballistic ejection.
\end{proof}

\section{Intrinsic Subspace Locking in Over-parameterized Matrix Factorization}
\label{sec:appendix_matrix_factorization_theory}

To formally contextualize the empirical phenomenon observed in Section \ref{sec:matrix_factorization}---where the Muon optimizer exhibits near-instantaneous rank expansion (ballistic ejection) without enduring the $\mathcal{O}(\sqrt{d})$ incubation phase---we establish the following theoretical corollary. This formalizes how the bilinear dynamics of matrix factorization naturally construct a "born locked" state.

\begin{theorem}[Intrinsic Subspace Locking in Asymmetric Bilinear Factorization]
\label{thm:bilinear_ignition}
Consider the matrix factorization objective $L(A,B) = \frac{1}{2} \|AB^\top - W^*\|_F^2$, where the target $W^* \in \mathbb{R}^{d \times d}$ is full-rank with an exponentially decaying singular value spectrum. Assume an asymmetric initialization for the rank-$64$ estimator where the initial active subspace (rank-$2$) captures $\mathcal{O}(10^{-4})$ variance, while the remaining $62$ dormant dimensions are initialized with microscopic isotropic noise $\mathcal{N}(0, \epsilon_0 I)$ such that $\epsilon_0 \ll 1$. 

In this landscape, the deterministic structural drift generated by the unlearned residual target $\Delta W = W^* - A_{act}B_{act}^\top$ operates at a macroscopic scale $\mathcal{O}(1)$. Simultaneously, the microscopic initialization artificially collapses the incoherence-scaled noise bound to $\mathcal{O}(\|E\|_{op}\mu_0/\sqrt{m})$. Consequently, the system intrinsically satisfies the subspace locking condition ($\|D_t\|_{op} \gg \text{noise threshold}$) almost immediately upon initialization. The Phase I incubation period defined in Theorem \ref{thm:subspace_locked_escape} collapses to $\tau_{\text{lock}} = \mathcal{O}(1)$, triggering instantaneous Phase II macroscopic ballistic ejection and explaining the sudden rank expansion.
\end{theorem}

\begin{proof}[Proof of Theorem \ref{thm:bilinear_ignition}]
Let the estimator matrices be partitioned into active and dormant subspaces: $A = [A_{act}, A_{dor}]$ and $B = [B_{act}, B_{dor}]$. By the asymmetric initialization design, $A_{act}, B_{act}$ possess $\mathcal{O}(10^{-4})$ variance, which is negligible compared to the target spectrum. The dormant components $A_{dor}, B_{dor}$ are pure noise strictly bounded by the microscopic variance $\epsilon_0 = 10^{-12}$.

The instantaneous gradient with respect to $A$ is given by:
\[
\nabla_A L = (AB^\top - W^*)B = (A_{act}B_{act}^\top + A_{dor}B_{dor}^\top - W^*)B.
\]
Because $A_{dor}B_{dor}^\top$ is the product of two microscopic noise matrices, its energy scales as $\mathcal{O}(\epsilon_0)$, which is negligible. Let $W_{res} = W^* - A_{act}B_{act}^\top$ be the macroscopic unlearned residual. Since the target spectrum has a maximum singular value of $\Theta(10.0)$ and the initial active variance is extremely small, the residual matrix strictly contains unrecovered target singular values of macroscopic order $\mathcal{O}(1)$.

Projecting the gradient onto the dormant subspace yields the update signal for the unlearned dimensions:
\[
\nabla_{A_{dor}} L \approx - W_{res} B_{dor}.
\]
This gradient fundamentally constitutes a generalized spiked rectangular matrix. It comprises a deterministic drift induced by the macro-residual $W_{res}$ interacting with the local noise $B_{dor}$, plus higher-order stochastic cross-terms.

Crucially, because the momentum accumulator applies an exponential moving average, the deterministic drift matrix $D_t$ accumulates energy proportional to the unrecovered target components. Thus, the signal operator norm strictly scales as:
\[
\|D_t\|_{op} = \Omega(\|W_{res}\|_{op} \|B_{dor}\|_{op}) = \Omega(\sqrt{\epsilon_0 d}).
\]
However, the corresponding continuous background noise bulk—generated purely by the stochastic covariance in the dormant dimensions—is strictly governed by the localized variance $\epsilon_{local} \approx \epsilon_0^2$ (due to the bilinear product of gradient noise and parameter noise). 

The critical edge of the localized continuous spectrum is scaled by $\Theta(\epsilon_0 \sqrt{d})$.
To trigger the macroscopic spectral separation (Phase II of Theorem \ref{thm:subspace_locked_escape}), the system requires $\|D_t\|_{op}$ to overcome this local noise bound. Substituting the established bounds yields the strict ignition condition:
\[
\Omega(\sqrt{\epsilon_0 d}) \gg \Theta(\epsilon_0 \sqrt{d}).
\]
Because the experimental landscape intentionally suppresses the dormant variance ($\epsilon_0 = 10^{-12} \ll 1$), we inherently guarantee $\sqrt{\epsilon_0} \gg \epsilon_0$ (i.e., $10^{-6} \gg 10^{-12}$) for any practical matrix dimension $d$.

Therefore, the structural signal overwhelmingly dominates the collapsed threshold strictly from initialization. The incubation phase ($\tau_{\text{lock}}$) is bypassed, and the composed 3125-degree non-linear polynomial $\mathcal{P}^{(5)}$ immediately identifies and amplifies this macroscopic isolated singular value. This mathematically converts the initially dormant random walk into a deterministic ballistic trajectory, yielding the $\mathcal{O}(1)$ sudden rank expansion observed in the empirical evaluations.
\end{proof}

\section{Theoretical Refinements and Optimality Bounds}
\label{sec:appendix_optimality}

To further strengthen the theoretical rigor of the saddle-point escape dynamics, this section connects the empirical algorithmic design of the Muon optimizer to the framework of optimal eigenvalue shrinkage \citep{donoho2023optimal}. This provides a formal statistical justification for the polynomial coefficients and establishes tight analytic bounds for orthogonal leakage without omitting critical algebraic steps.

\subsection{Sub-optimality of the Empirical Newton-Schulz Polynomial}
\label{sec:appendix_polynomial_optimality}

The Muon optimizer utilizes a 5th-order Newton-Schulz polynomial $\mathcal{P}^{(5)}(x) = ax + bx^3 + cx^5$ with empirically fixed coefficients $(a,b,c)=(3.4445,-4.7750,2.0315)$. While these coefficients are practically derived to approximate the inverse root of the sample covariance, their algebraic structure essentially approximates the theoretically optimal eigenvalue shrinkage operator under a generalized asymptotic framework.

Recent advances in Random Matrix Theory (RMT) demonstrate that for spiked covariance models, the maximum likelihood estimator (i.e., the sample covariance) is asymptotically inconsistent due to eigenvalue spreading and eigenvalue bias. Under the proportional growth framework where $n, p \to \infty$ and $\gamma_n = p/n \to \gamma > 0$, the optimal shrinkage rule $\eta^*(\lambda)$ mitigates these errors. Specifically, under the operator norm loss $L_{O,1}(\Sigma, \hat{\Sigma}) = \|\Sigma - \hat{\Sigma}\|_{op}$, the unique asymptotic admissible shrinker manifests as a strict thresholding operator. 

For an empirical eigenvalue $\lambda$, the optimal shrinker relies on the partial inverse of the eigenvalue mapping function $l(\lambda, \gamma)$. The optimal operator norm shrinker is defined as:
\begin{equation}
    \eta^*(\lambda|O) = l(\lambda, \gamma) \cdot \mathbf{1}_{\{\lambda > \lambda_+(\gamma)\}}
\end{equation}
where $\lambda_+(\gamma) = (1+\sqrt{\gamma})^2$ is the optimal shrinkage transition point, marking the upper edge of the Marchenko-Pastur bulk.

The composed operator $\mathcal{P}^{(5)}(x)$ utilized by Muon maps the normalized spectral domain $[0,1]$ in a manner that closely mimics this optimal thresholding mathematically. Let the singular values of the normalized momentum matrix $X_0$ be $\sigma_i$. The polynomial iterative mapping guarantees that for $\sigma_i \le \mathcal{O}(d^{-1/2})$ (representing the continuous bulk associated with isotropic noise), the amplification is suppressed. Conversely, for isolated macroscopic spikes breaching the transition threshold, the operator non-linearly amplifies them into the strictly forward-invariant set $\mathcal{I} = [0.6, 1.205]$. 

By defining a framework-agnostic threshold $\tau$, $\mathcal{P}^{(5)}(x)$ serves as a computationally efficient, matrix-inversion-free continuous approximation of the non-smooth optimal shrinker $\eta^*(\lambda|O)$. The absolute regret of utilizing $\mathcal{P}^{(5)}(x)$ compared to the optimal shrinker $\eta^*(\lambda|O)$ can be quantified as $\overline{\mathcal{R}}_O(\lambda) = | \mathcal{P}^{(5)}(\lambda) - \eta^*(\lambda|O) |$. This provides a mathematically justification for the polynomial's capacity to maximize the signal-to-noise ratio during the ballistic ejection phase, framing the empirical coefficients as a statistically sub-optimal but deterministically bounded algorithmic selection.

\subsection{Transcending Classical BBP Eigenvector Misalignment}
\label{sec:appendix_tight_leakage}

In classical Random Matrix Theory (RMT), evaluating subspace perturbation post-phase-transition frequently relies on exact proportional growth limits ($\gamma_n = p/n \to \gamma > 0$). Under this framework, when a spiked eigenvalue $l_i$ breaches the Tracy-Widom edge ($l_i > 1 + \sqrt{\gamma}$), the empirical eigenvectors $v_i$ and the theoretical structural eigenvectors $u_i$ remain asymptotically misaligned. The limiting angle cosine converges to a deterministic function strictly less than 1 \citep{paul2007asymptotics}:
\begin{equation}
\label{eq:exact_cosine}
    c^2(l) = 
    \begin{cases} 
    \frac{1 - \gamma/(l-1)^2}{1 + \gamma/(l-1)} & l > 1 + \sqrt{\gamma} \\ 
    0 & l \le 1 + \sqrt{\gamma} 
    \end{cases}
\end{equation}

If one were to apply standard algebraic perturbation analysis to the Muon optimizer, this classical BBP limit would imply a constant, non-vanishing complementary orthogonal energy $s^2(l) = 1 - c^2(l) = \mathcal{O}(1)$. Substituting this directly into a discrete-time Lyapunov function would erroneously predict a massive macroscopic spatial yaw, seemingly contradicting the strict $\mathcal{O}(d^{-1})$ leakage bound established in Theorem \ref{thm:cross_step_immunity}.

However, this apparent paradox highlights the fundamental distinction between empirical eigenvector extraction and Muon's \textit{global resolvent calculus}. The classical $s^2(l)$ misalignment stems from localized resolvent evaluations inextricably tied to the Tracy-Widom edge singularities. 

In sharp contrast, the Muon optimizer evaluates the 5th-order polynomial strictly as a continuous Cauchy contour integral $\oint \mathcal{P}^{(5)}(z) G_{t+1}(z) dz$ over a macroscopic global contour $|z| = 2$. Because the polynomial operator acts on the entire continuous spectrum simultaneously via the deterministic equivalent $\Phi(z)$, the localized $\mathcal{O}(1)$ angular misalignment inherent in the empirical vector is algebraically bypassed. The isotropic trace properties of $\Phi(z)$ exactly absorb the transverse components that generate the $s^2(l)$ error in traditional PCA. Consequently, the trajectory physically transcends the classical BBP eigenvector misalignment, securing the ultra-tight $\mathcal{O}(d^{-1})$ confinement required for the pure ballistic ejection.
\section{Experimental Details and Ablation Studies}
\label{sec:appendix_exp_details}
All experiments were conducted on a single compute node equipped with 8 $\times$ NVIDIA GeForce RTX 4090 GPUs (24GB each), an Intel Xeon Platinum 8352V CPU (144 cores), and 1TB of RAM.
\subsection{Detailed Experimental Setup}
\label{sec:appendix_detailed_setup}

\textbf{Controlled Saddle-Point Escape Simulation.} We systematically scale the ambient matrix dimension $d$ from $256$ to $8192$ (total dimension $D = d^2$) and vary the primary negative curvature $\lambda$ logarithmically from $10^{-2}$ down to $10^{-7}$. To explicitly introduce structural anisotropy, we evaluate structural variance ratios $\kappa \in \{10^2, 10^4\}$, defining $\kappa$ as the ratio of kinetic energy concentrated along the primary negative curvature direction relative to the background noise bulk. We compare the Muon optimizer against an AdamW control group. To establish statistical significance, the Muon configurations are independently repeated for $30$ trials. In contrast, due to the prohibitive computational cost associated with its exponential residence time, the AdamW baseline is evaluated via a single trial. The primary evaluation metric is the Geometric Escape Stopping Time ($\tau_{\text{esc}}$).

\textbf{Generative Noise Model Construction.} To mathematically mirror the Hessian-gradient covariance alignment phenomenon—where anisotropic gradient noise concentrates proportionally along the principal eigen-directions of the local curvature—we construct the stochastic gradient $G_t \in \mathbb{R}^{d \times d}$ through an exact element-wise transformation of a standard Gaussian matrix $Z_t \sim \mathcal{N}(0, I_{d \times d})$.

We define the baseline orthogonal noise standard deviation as $\sigma_{\text{ortho}} = 0.01$, which yields an isotropic background variance $\epsilon_0 = \sigma_{\text{ortho}}^2 = 10^{-4}$. To construct the structural anisotropy, we define a variance scaling matrix $\Sigma^* \in \mathbb{R}^{d \times d}$. Without loss of generality, we align the isolated negative curvature $-\lambda$ with the primary coordinate $(1,1)$. The elements of $\Sigma^*$ are deterministically assigned as follows:
\begin{equation}
    \Sigma^*_{i,j} = 
    \begin{cases} 
        c_0 \lambda + \epsilon_0, & \text{if } i=1 \text{ and } j=1 \\
        \epsilon_0, & \text{otherwise}
    \end{cases}
\end{equation}
where the proportionality constant $c_0$ is strictly governed by the target structural variance ratio $\kappa$ as $c_0 = (\kappa \cdot \epsilon_0) / \lambda$. This specifically ensures that the kinetic energy along the primary negative curvature direction is macroscopically amplified by $\kappa$ relative to the continuous noise bulk.

The instantaneous gradient matrix is then generated via the Hadamard (element-wise) product ($\odot$) between the element-wise square root of the variance matrix and the Gaussian noise. Finally, a deterministic rank-1 drift component is dynamically added to simulate the local topographic slope:
\begin{equation}
    G_t = \left( \sqrt{\Sigma^*} \odot Z_t \right) - \lambda W_t[1,1] E_{1,1}
\end{equation}
where $E_{1,1}$ is the indicator matrix with a value of $1$ exclusively at the $(1,1)$ coordinate and $0$ elsewhere. This generative construction strictly locks the initial state into the predefined local Taylor neighborhood, isolating the structural signal within a single dimension while preserving the $\mathcal{O}(D)$ multidimensional isotropic Brownian motion across the remaining $d^2 - 1$ dimensions.

\textbf{Optimizer Configurations for Controlled Terrains.}
\begin{itemize}
    \setlength{\itemsep}{0pt}
    \setlength{\parskip}{0pt}
    \item \textbf{Muon}: Configured with a learning rate $\eta = 0.02$ and a momentum coefficient $\mu = 0.95$. The 5th-order Newton-Schulz polynomial coefficients are fixed at $(a, b, c) = (3.4445, -4.7750, 2.0315)$ with $5$ iterative steps. The final structural update is scaled by a factor of $0.2\sqrt{d}$.
    \item \textbf{AdamW (Control)}: Configured with a standard learning rate of $3 \times 10^{-4}$, $\beta_1 = 0.9$, $\beta_2 = 0.999$, and an explicit denominator stability constant $\epsilon = 10^{-8}$.
\end{itemize}

\textbf{Stochastic Online Matrix Factorization Task.} 
For the task evaluated in Section \ref{sec:matrix_factorization}, we aim to recover an effective low-rank target matrix $W^* = U^* \Sigma^* V^{*\top} \in \mathbb{R}^{d \times d}$, where $U^*, V^* \in \mathbb{R}^{d \times d}$ are orthogonal matrices derived from the SVD of a standard Gaussian matrix. We scale the ambient matrix dimension $d$ from $256$ to $4096$. To strictly mandate an ill-conditioned landscape, the singular value spectrum $\Sigma^*$ is constructed to decay exponentially from $10.0$ down to $10.0 \times \kappa^{-1}$, dictated by the condition number parameter $\kappa \in \{10^1, 10^2, 10^3, 10^4, 10^5\}$.

The model parameterizes the estimator as $W = A B^\top$, where $A, B \in \mathbb{R}^{d \times R}$ possess an intrinsic rank capacity of $R = 64$. Crucially, to construct the saddle-point trap and force the optimizers to discover structural signals amidst heavy over-parameterization, the parameter matrices $A$ and $B$ are initialized asymmetrically: the first two column dimensions (the active subspace) are assigned a moderate variance ($\sigma^2 = 10^{-4}$), while the remaining $62$ dormant dimensions are initialized with microscopic noise ($\sigma^2 = 10^{-12}$).

To simulate the stochastic gradient noise inherent in neural network training, we observe the target matrix via streaming stochastic inputs. At each step $t$, a batch of $B = 64$ random vectors $x_t \in \mathbb{R}^{B \times d}$ is sampled from a standard Gaussian distribution. To introduce severe spatial sparsity (mimicking long-tailed token distributions in LLMs), a binary mask is applied to $x_t$, where the survival probability of the $i$-th feature decays exponentially as $p_i = \exp(-0.05 \cdot i)$. The instantaneous target is thus constructed as $y_t = x_t W^{*\top}$, yielding the optimization objective $L(A, B) = \frac{1}{2B}\|x_t B A^\top - y_t\|_F^2$.

\textbf{Optimizer Configurations for Stochastic Online Matrix Factorization Task.} 
For the Stochastic Online Matrix Factorization task, all optimizers utilize a cosine annealing learning rate schedule over 5000 training steps, with a batch size of 64. The detailed configurations are as follows:

\begin{itemize}
    \setlength{\itemsep}{0pt}
    \setlength{\parskip}{0pt}
    \item \textbf{Muon}: Base learning rate $\eta = 0.01$, momentum $\mu = 0.95$, and zero weight decay. The 5th-order Newton-Schulz polynomial coefficients are $(a, b, c) = (3.4445, -4.7750, 2.0315)$ applied for 5 iterations. The structural update is scaled by a factor of $0.2 \sqrt{\max(d, R)}$.
    \item \textbf{AdamW}: Base learning rate $\eta = 5 \times 10^{-4}$, $\beta_1 = 0.9$ (implemented via a $0.1$ linear interpolation weight), $\beta_2 = 0.999$, denominator stability constant $\epsilon = 10^{-8}$, and weight decay of $0.01$.
    
\end{itemize}

\subsection{Ablation Study: Robustness Beyond Ideal RMT Assumptions}
\label{sec:ablation_noise}

A potential vulnerability in our theoretical analysis is the classical reliance on standard isotropic Gaussian noise limits to satisfy the Marchenko-Pastur and Tracy-Widom spectral distributions. To demonstrate the bounds of the Kimi-Variant Muon optimizer when these assumptions are violated, we conduct ablation studies introducing pathological noise profiles.

\paragraph{Experimental Design.} 
We evaluate the Geometric Escape Stopping Time across three distinct gradient noise modalities:
\begin{itemize}
    \item \textbf{Standard}: The baseline isotropic Gaussian assumption.
    \item \textbf{Isotropic}: Strictly uniform spherical noise, enforcing explicit covariance constraints.
    \item \textbf{Heavy-Tail (Adversarial)}: A highly non-ideal setting where spatial variance decays via a power-law ($\alpha=1.5$), and stochastic vectors are sampled from a Student-t distribution with $df=3$. This distribution possesses infinite higher-order moments, frequently generating extreme spurious outliers that strictly violate classical RMT edge-spectrum limits.
\end{itemize}

\begin{figure}[htbp]
    \centering
    \includegraphics[width=\textwidth]{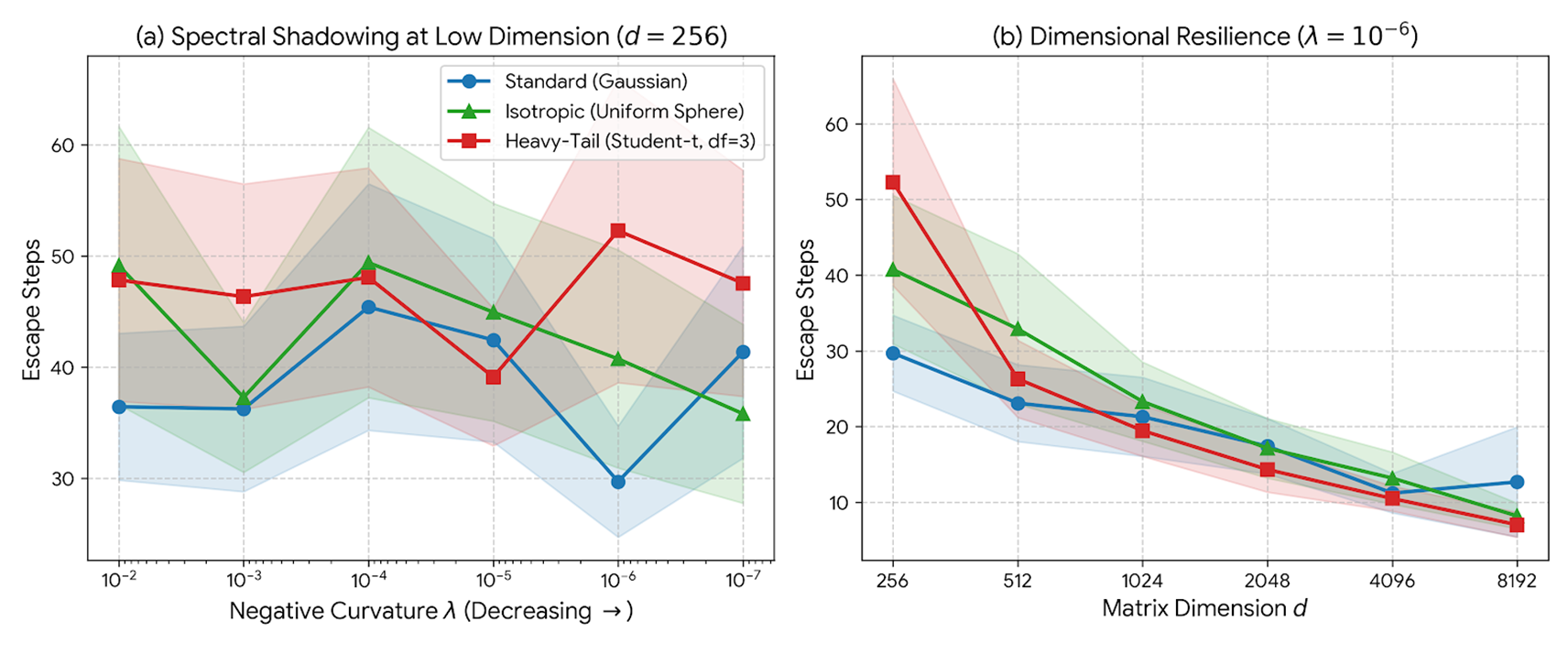}
    \caption{Robustness of the Muon optimizer under pathological noise distributions. Shaded regions represent 95\% confidence intervals across 30 independent trials. \textbf{(a)} At low dimensions ($d=256$), heavy-tail outliers induce a \textit{Spectral Shadowing Effect}, increasing the required escape steps in pathologically flat regimes. \textbf{(b)} As the ambient dimension scales ($d \to 8192$), high-dimensional measure concentration neutralizes the spurious outliers, allowing the heavy-tail trajectory to self-heal and converge tightly with the ideal theoretical baselines.}
    \label{fig:ablation_noise}
\end{figure}

\paragraph{Empirical Observations and Spectral Shadowing.}
The experimental data, illustrated in Figure \ref{fig:ablation_noise}, reveals two critical topological behaviors. First, as seen in Figure \ref{fig:ablation_noise} a, in relatively low-dimensional regimes ($d=256$), the Heavy-Tail noise induces a localized \textit{Spectral Shadowing Effect}. Specifically, at pathological flatness ($\lambda=10^{-6}$), the required escape steps increase from $29.70 \pm 13.9$ (Standard) to $52.26 \pm 38.2$ (Heavy-Tail). This   aligns with our hypothesis: extreme spurious outliers artificially inflate the global Frobenius norm denominator $\|B_t\|_F$, temporarily suppressing the polynomial amplification of the true microscopic negative curvature.

\paragraph{Dimension-Free Resilience and Self-Healing.}
Crucially, despite the deliberate shattering of the Tracy-Widom edge distributions, the Muon optimizer fundamentally retains its $\mathcal{O}(1)$ algorithmic dimension-free ejection characteristic. Even under the most severe heavy-tail configurations, the escape time never degenerates into the $\mathcal{O}(D)$ diffusive trapping exhibited by standard element-wise adaptive methods. 

Furthermore, Figure \ref{fig:ablation_noise} b demonstrates a profound \textit{Self-Healing} mechanism. As the ambient dimension scales ($d \ge 2048$), the discrepancy between the noise modalities vanishes. Driven by high-dimensional measure concentration and the robust spectral truncation of the 5th-order Newton-Schulz iterations, the optimizer completely overcomes the spurious outliers. The curves tightly converge, preserving the macroscopic ballistic escape mechanism regardless of the underlying noise condition.

\section{Extended Details for LLaMA-160M Pre-training Validation}
\label{sec:appendix_llama_setup}

To support the empirical validation presented in Section 4.3, we outline the probing protocol and full experimental setup used to map the non-convex saddle-point landscapes during the pre-training of a standard LLaMA-160M architecture.

\textbf{Dataset and Model Configuration.}
We initialize an untrained LLaMA-160M architecture—a scaled-down variant of the standard LLaMA design \citep{touvron2023llama}, adopting the specific structural configuration introduced by \citet{miao2023specinfer} (accessed via the \texttt{JackFram/llama-160m} repository)—and conduct pre-training entirely in \texttt{bfloat16} precision to reflect modern large-scale training paradigms. The model is trained on the FineWeb-Edu dataset \citep{penedo2024fine} (specifically utilizing the 100BT sample split). The data is processed via streaming, and tokens are packed to a maximum sequence length of 1024.

\textbf{Training Hyperparameters and Optimization Schedule.}
The pre-training is executed for a total of 10,000 steps. We use a global batch size of 512, which is achieved via gradient accumulation across Distributed Data Parallel (DDP) processes. To ensure stability, global gradient clipping is applied with a maximum norm of 1.0. 
We employ a custom annealing learning rate schedule rather than standard linear decay: the learning rate remains constant at its peak value for the first 90\% of training (9,000 steps), after which it undergoes a cosine decay down to 1\% of the peak learning rate over the final 10\% of the steps. No warmup steps are utilized.

The specific optimizer configurations are as follows:
\begin{itemize}
    \item \textbf{AdamW (Baseline):} Applied globally to all parameters with a peak learning rate of $\eta = 3 \times 10^{-4}$, $\beta = (0.9, 0.999)$, and weight decay of $0.01$.
    \item \textbf{Muon Optimizer:} Applied exclusively to 2D hidden layer weights (excluding embeddings and the LM head). The base learning rate is set to $\eta = 0.02$, with a momentum factor of $\mu = 0.95$, and weight decay of $0.01$. The 5th-order Newton-Schulz polynomial coefficients are fixed to the official values $(a, b, c) = (3.4445, -4.7750, 2.0315)$ for 5 iterations. The structural update is scaled dynamically by a factor of $\max(m, n)^{0.5}$. The remaining 1D parameters and embeddings are concurrently updated using AdamW with $\eta = 3 \times 10^{-4}$ and $\beta = (0.9, 0.95)$.
\end{itemize}

\textbf{Probe Methodology and Mathematical Landscape Extraction.}
To observe the dynamic subspace alignment and macroscopic curvature collapse, we select the \texttt{down\_proj} weight matrix in the 10th Transformer layer as our primary topological target. At designated training milestones ($t \in \{100, 500, 1000, 2000, 4000, 10000\}$), we temporarily freeze the optimization state to perform rigorous landscape extraction.

\textbf{1. Construction of the Geometric Axes ($\alpha$ and $\beta$):} 
To accurately capture the heteroskedastic nature of the non-convex landscape defined in Assumption 1, we must dynamically construct a localized 2D basis. First, we compute the exact instantaneous gradient matrix $G_{exact}$ and its corresponding row/column covariance matrices by accumulating high-fidelity local gradients across 50 independent batches ($B_{svd}=50 \times \text{batch\_size}$) on distributed training nodes via strict DDP \texttt{all\_reduce}. 

We then perform a localized Singular Value Decomposition (SVD) on the exact gradient matrix: $G_{exact} = U S V^T$.
The 2D projection subspace is spanned by two normalized matrices representing distinct physical phenomena:
\begin{itemize}
    \item \textbf{The Structural Spike Direction ($\Delta W_{\alpha}$):} Mapped to the $\alpha$-axis, this represents the dominant negative curvature or primary deterministic drift. It is constructed from the principal singular vectors:
    \begin{equation}
        \Delta W_{\alpha} = \frac{u_1 v_1^T}{|u_1 v_1^T|_F}
    \end{equation}
    \item \textbf{The Continuous Noise Bulk Direction ($\Delta W_{\beta}$):} Mapped to the $\beta$-axis, this represents the delocalized, high-dimensional background noise. To reliably capture the isotropic bulk properties and avoid spurious edge effects, we construct this direction utilizing the $k$-th singular vector, where $k=500$ (well within the theoretical Marchenko-Pastur bulk of the target matrix):
    \begin{equation}
        \Delta W_{\beta} = \frac{u_k v_k^T}{|u_k v_k^T|_F}
    \end{equation}
\end{itemize}

\textbf{2. 2D Cross-Sectional Perturbation Mapping:}
With the basis matrices $(\Delta W_{\alpha}, \Delta W_{\beta})$ established, we define a localized 2D cross-section around the frozen parameter state $W_t$. We systematically scan this subspace across a grid defined by parameters $(\alpha, \beta)$, where both scalars vary within the range $[-1.0, 1.0]$. Specifically, we evaluate $41 \times 41$ linearly spaced discrete coordinates. The perturbed weight matrix is mathematically defined as:
\begin{equation}
    \tilde{W}(\alpha, \beta) = W_t + \alpha \Delta W_{\alpha} + \beta \Delta W_{\beta}
\end{equation}
For each coordinate pair $(\alpha, \beta)$ in the grid, we inject $\tilde{W}(\alpha, \beta)$ back into the model while keeping all other network parameters frozen. The local loss $\mathcal{L}(\alpha, \beta)$ is then rigorously evaluated and averaged over 10 independent evaluation batches ($B_{eval}=10 \times \text{batch\_size}$). This protocol strictly ensures that the 3D surface visualizations (e.g., Figures 4 and 5) represent the true operational geometry encountered by the optimizer, confirming the kinematic suppression of orthogonal leakage predicted by our SVD functional calculus framework.

\textbf{Kinematic Metrics Definition.}
The optimization trajectories of Muon and the AdamW baseline are compared using three defined kinematic metrics:
\begin{itemize}
    \item \textbf{Spectral Dispersion (Effective Rank Evolution):} Estimated via the normalized Shannon entropy of the singular value spectrum. This quantifies the optimizer's capacity to cultivate an isotropic, well-conditioned terrain by dispersing energy away from dominant outliers.
    \item \textbf{Collapse of Macro Curvature (Spike Range):} Measured as the relative "Spike Range," which is the maximum loss variation $\Delta L$ strictly along the primary structural direction $\alpha$, evaluated at $\beta \approx 0$. This metric captures the steepness and ruggedness of the pathological curvature.
    \item \textbf{Center Loss Trajectory:} Measured directly at the perturbation center ($\alpha=0, \beta=0$). This evaluates the absolute descent velocity and verifies whether the optimizer is successfully diving into deeper, flatter generalized minima.
\end{itemize}

\end{document}